%% file: main.tex
\documentclass{article}
\usepackage[utf8]{inputenc}

\usepackage[utf8]{inputenc}
\usepackage[margin=1in]{geometry}
\usepackage{mathtools}
\usepackage{amsmath}
\usepackage{amsfonts}
\usepackage{amssymb}
\usepackage{textcomp}
\usepackage{pgfplots}
\usepackage{graphicx}
\usepackage{wrapfig}
\usepackage[]{youngtab}
\usepackage{amsthm}
\usepackage{tikz}
\usepackage{tikz-cd}
\usepackage{stmaryrd}
\usepackage{enumerate}
\usepackage{bbm}
\usepackage{float}

\usepackage{eucal}
\usepackage{mathrsfs}
\pgfplotsset{width=10cm,compat=1.9}
\usepackage{natbib}
\usepackage{hyperref}
\hypersetup{
    colorlinks=true,
    linkcolor=blue,
    citecolor=blue
}
\usepackage[capitalize]{cleveref}
\usepackage{caption}
\usepackage{subcaption}

\newtheorem{theorem}{Theorem}
\newtheorem{lemma}[theorem]{Lemma}
\newtheorem{proposition}[theorem]{Proposition}

\newtheorem{corollary}[theorem]{Corollary}
\newtheorem{definition}[theorem]{Definition}

\newcommand{\pp}{\mathbb{P}}
\newcommand{\F}{\mathcal{F}}
\newcommand{\norm}[1]{\left|\left| #1 \right|\right|}
\newcommand{\abs}[1]{\left| #1 \right|}
\DeclareMathOperator{\supp}{supp}

\newcommand{\rr}{\mathbb{R}}
\newcommand{\ee}{\mathbb{E}}

\newcommand{\cs}{\mathcal{S}}
\DeclareMathOperator{\diam}{diam}
\DeclareMathOperator{\med}{Med}
\DeclareMathOperator{\vol}{vol}
\DeclareMathOperator*{\argmin}{argmin}
\renewcommand{\epsilon}{\varepsilon}
\renewcommand{\phi}{\varphi}
\newcommand{\ric}{Ric}

\title{Intrinsic Dimension Estimation Using Wasserstein Distances}
\author{Adam Block \\ MIT \and Zeyu Jia \\ MIT \and Yury Polyanskiy \\ MIT \and Alexander Rakhlin \\
MIT }

\date{}

\begin{document}

\maketitle
\def\yp#1{\textcolor{red}{YP: #1}}
\begin{abstract}
    It has long been thought that high-dimensional data encountered in many practical machine learning tasks have low-dimensional structure, i.e., the manifold hypothesis holds.  A natural question, thus, is to estimate the intrinsic dimension of a given population distribution from a finite sample.  We introduce a new estimator of the intrinsic dimension and provide finite sample, non-asymptotic guarantees.  We then apply our techniques to get new sample complexity bounds for Generative Adversarial Networks (GANs) depending only on the intrinsic dimension of the data.
\end{abstract}
\section{Introduction}
\input{first_body}

\section{Dimension Estimation}\label{sec:dimestim}

\input{secDimensionEstimation}

\input{secLearning}

\section*{Acknowledgements}
We acknowledge support from the NSF through award DMS-2031883 and from the Simons Foundation through Award 814639 for the Collaboration on the Theoretical Foundations of Deep Learning. We acknowledge the support from NSF under award DMS-1953181, NSF Graduate Research Fellowship support under Grant No. 1122374, and support from the MIT-IBM Watson AI Lab.

\bibliographystyle{authordate1}
\bibliography{References.bib}
\appendix

\input{proof-coverings}

\input{proof-dimestim}

\input{proof-metric}

\input{proof-miscellany}

\end{document}

%% file: first_body.tex
Recently, practical applications of machine learning involve a very large number of features, often many more than there are samples on which to train a model.  Despite this imbalance, many modern machine learning models work astonishingly well.  One of the more compelling explanations for this behavior is the manifold hypothesis, which posits that, though the data appear to the practitioner in a high-dimensional, ambient space, $\rr^D$, they really lie on (or close to) a low dimensional space $M$ of ``dimension" $d \ll D$, where we define dimension formally below.  A good example to keep in mind is that of image data: each of thousands of pixels corresponds to three dimensions, but we expect that real images have some inherent structure that limits the true number of degrees of freedom in a realistic picture.  This phenomenon has been thoroughly explored over the years, beginning with the linear case and moving into the more general, nonlinear regime, with such works as \cite{niyogi2008finding,niyogi2011topological,belkin2001laplacian,bickel2007local,levina2004maximum,kpotufe2011k,kpotufe2012tree,kpotufe2013adaptivity,weed2019sharp, tenenbaum2000,Bernstein00,kim2019minimax,farahmand2007manifold}, among many, many others.  Some authors have focused on finding representations for these lower dimensional sets \citep{niyogi2008finding,belkin2001laplacian,tenenbaum2000,roweis2000nonlinear,donoho2003hessian}, while other works have focused on leveraging the low dimensionality into statistically efficient estimators \citep{bickel2007local,kpotufe2011k,nakada2020adaptive,kpotufe2012tree,kpotufe2013adaptivity,ashlagi2021functions}.

In this work, our primary focus is on estimating the intrinsic dimension.  To see why this is an important question, note that the local estimators of \cite{bickel2007local,kpotufe2011k,kpotufe2013adaptivity} and the neural network architecture of \cite{nakada2020adaptive} all depend in some way on the intrinsic dimension.  As noted in \cite{levina2004maximum}, while a practitioner may simply apply cross-validation to select the optimal hyperparameters, this can be very costly unless the hyperparameters have a restricted range; thus, an estimate of intrinsic dimension is critical in actually applying the above works.  In addition, for manifold learning, where the goal is to construct a representation of the data manifold in a lower dimensional space, the intrinsic dimension is a key parameter in many of the most popular methods \citep{tenenbaum2000,belkin2001laplacian,donoho2003hessian,roweis2000nonlinear}.

We propose a new estimator, based on distances between probability distributions, as well as provide rigorous, finite sample guarantees for the quality of the novel procedure.  Recall that if $\mu, \nu$ are two measures on a metric space $(M, d_M)$, then the Wasserstein-$p$ distance between $\mu$ and $\nu$ is
\begin{equation}\label{eq:wassersteindef}
    W_p^M(\mu, \nu)^p = \inf_{(X, Y) \sim \Gamma(\mu, \nu)} \ee\left[d_M(X, Y)^p \right]
\end{equation}
where $\Gamma(\mu, \nu)$ is the set of all couplings of the two measures.  If $M \subset \rr^D$, then there are two natural metrics to put on $M$: one is simply the restriction of the Euclidean metric to $M$ while the other is the geodesic metric in $M$, i.e., the minimal length of a curve in $M$ that joins the points under consideration.  In the sequel, if the metric is simply the Euclidean metric, we leave the Wasserstein distance unadorned to distinguish it from the intrinsic metric.  For a thorough treatment of such distances, see \cite{villani2008optimal}.  We recall that the H{\"o}lder integral probability metric (H{\"o}lder IPM) is given by
\begin{equation*}\label{eq:holderipm}
    d_{\beta, B}(\mu, \nu) = \sup_{f \in C_B^\beta(\Omega)} \ee_\mu[f(X)] - \ee_\nu[f(Y)]
\end{equation*}
where $C_B^\beta(\Omega)$ is the H{\"o}lder ball defined in the sequel.  When $p = \beta = 1$, the classical result of Kantorovich-Rubinstein says that the Wasserstein and H{\"o}lder distances agree.  It has been known at least since \cite{dudley1969speed} that if a space $M$ has dimension $d$, $\pp$ is a measure with support $M$, and $P_n$ is the empirical measure of $n$ independent samples drawn from $\pp$, then $W_1^M(P_n, \pp) \asymp n^{- \frac 1d}$.  More recently, \cite{weed2019sharp} has determined sharp rates for the convergence of this quantity for higher order Wasserstein distances in terms of the intrinsic dimension of the distribution.  Below, we find sharp rates for the convergence of the empirical measure to the population measure with respect to the H{\"o}lder IPM: if $\beta < \frac d2$, then $d_\beta(P_n, \pp) \asymp n^{- \frac \beta d}$ and if $\beta > \frac d2$ then $d_\beta(P_n, \pp) \asymp n^{- \frac 12}$.  These sharp rates are intuitive in that convergence to the population measure should only depend on the intrinsic complexity (i.e. dimension) without reference to the possibly much larger ambient dimension.

The above convergence results are nice theoretical insights, but they have practical value, too.  The results of \cite{dudley1969speed,weed2019sharp}, as well as our results on the rate of convergence of the H{\"o}lder IPM, present a natural way to estimate the intrinsic dimension: take two independent samples, $P_n, P_{\alpha n}$ from $\pp$ and consider the ratio of $W_p^M(P_n, \pp) / W_p^M(P_{\alpha n}, \pp)$ or $d_\beta(P_n, \pp) / d_\beta(P_{\alpha n}, \pp)$; as $n \to \infty$, the first ratio should be about $\alpha^d$, while the second should be about $\alpha^{\frac \beta d}$, and so $d$ can be computed by taking the logarithm with respect to $\alpha$.  The first problem with this idea is that we do not know $\pp$; to address this, we instead compute the ratios using two independent samples.  A more serious issue regards how large $n$ must be in order for the asymptotic regime to apply.  As we shall see below, the answer depends on the geometry of the supporting manifold.

We define two estimators: one using the intrinsic distance and the other using Euclidean distance
\begin{align}\label{eq:d_n_estims}
    d_n = \frac{\log \alpha}{\log W_1(P_n, P_n') - \log W_1(P_{\alpha n}, P_{\alpha n}')} && \widetilde{d}_n = \frac{\log \alpha}{\log W_1^G(P_n, P_n') - \log W_1^G(P_{\alpha n}, P_{\alpha n}')}
\end{align}
where the primes indicate independent samples of the same size and $G$ is a graph-based metric that approximates the intrinsic metric.  Before we go into the details, we give an informal statement of our main theorem, which provides finite sample, non-asymptotic guarantees on the quality of the estimator\footnote{Explicit constants are given in the formal statement of Theorem \ref{thm:estimator}}:
\begin{theorem}[Informal version of Theorem \ref{thm:estimator}]
    Let $\pp$ be a measure on $\rr^D$ supported on a compact manifold of dimension $d$.  Let $\tau$ be the reach of $M$, an intrinsic geometric quantity defined below.  Suppose we have $N$ independent samples from $\pp$ where
    \begin{equation*}
        N = \Omega\left(\tau^{- d} \vee \left(\frac{\vol M}{\omega_d}\right)^{\frac{d + 2}{2 \gamma}} \vee \left(\log \frac 1\rho\right)^3 \right)
    \end{equation*}
    where $\omega_d$ is the volume of a $d$-dimensional Euclidean unit ball.  Then with probability at least $1 - 6 \rho$, the estimated dimension $\widetilde{d}_n$ satisfies
    \begin{equation*}
        \frac{d}{1 + 4 \gamma} \leq \widetilde{d}_n \leq (1 + 4 \gamma) d.
    \end{equation*}
    The same conclusion holds for $d_n$.
\end{theorem}
Although the guarantees for $d_n$ and $\widetilde{d_n}$ are similar, empirically $\widetilde{d}_n$ is much better, as explained below.  Note that the ambient dimension $D$ never enters the statistical complexity given above.  While the exponential dependence on the intrinsic dimension $d$ is unfortunate, it is likely necessary as described below.

While the reach, $\tau$, determines the sample complexity of our dimension estimator, consideration of the injectivity radius, $\iota$, is relevant for practical application.  Both geometric quantities are defined formally in the following section, but, to understand the intuition, note that, as discussed above, there are two natural metrics we could be placing on $M = \supp \pp$, the Euclidean metric and the geodesic distance.  The reach is, intuitively, the size of the largest ball with respect to the ambient metric such that we can treat points in $M$ as if they were simply in Euclidean space; the injectivity radius is similar, except it treats neighborhoods with respect to the intrinsic metric.  Considering that manifold distances are always at least as large as Euclidean distances, it is unsurprising that $\tau \lesssim \iota$.  Getting back to dimension estimation, specializing to the case of $\beta = p = 1$, and recalling \eqref{eq:d_n_estims}, there are now two choices for our dimension estimator: we could use Wasserstein distance with respect to the Euclidean metric or Wasserstein distance with respect to the intrinsic metric (which we will denote by $W_1^M$).  We will see that if $\iota \approx \tau$, then the two estimators induced by each of these distances behave similarly, but when $\iota \gg \tau$, the latter is better.  While we wish to use $W_1^M(P_n, P_n')$ to estimate the dimension, we do not know the intrinsic metric.  As such, we use the $k$NN graph to approximate this intrinsic metric and introduce the measure $W_1^G(P_n, P_n')$.  Note that if we had oracle access to geodesic distance $d_M$, then the $W_1^M$-based estimator $\widetilde{d}_n$ would only require $\asymp \iota^{-d}$ samples. However, our $k$NN estimator of $d_M$,  unfortunately, still requires the $\tau^{-d}$ samples.  Nevertheless, there is a practical advantage of $\widetilde d_n$ in that the metric estimator can leverage all $N = 2(1 + \alpha) n$ available samples, so that $\widetilde d_n$ works if $N \gtrsim \tau^{-d}$ and only $ n \gtrsim \iota^{-d}$, whereas for $d_n$ we require $n \gtrsim \tau^{-d}$ itself.

A natural question: is this more complicated approach necessary? i.e., is $\iota \gg \tau$ on real datasets?  We believe that the answer is yes.  To see this, consider the case of images of the digit 7 (for example) from MNIST \citep{MNIST}.  As a demonstration, we sample images from MNIST in datasets of size ranging in powers of 2 from $32$ to $2048$, calculate the Wasserstein distance between these two samples, and plot the resulting trend.  In the right plot, we pool all of the data to estimate the manifold distances, and then use these estimated distances to compute the Wasserstein distance between the empirical distributions.  In order to better compare these two approaches, we also plot the residuals to the linear fit that we expect in the asymptotic regime.  Looking at Figure \ref{fig1}, it is clear that we are not yet in the asymptotic regime if we simply use Euclidean distances; on the other hand, the trend using the manifold distances is much more clearly linear, suggesting that the slope of the best linear fit is meaningful.  Thus we see that in order to get a meaningful dimension estimate from practical data sets, we cannot simply use $W_1$ but must also estimate the geometry of the underlying distribution; this suggests that $\iota \gg \tau$ on this data manifold.  More generally, we note that the injectivity radius, $\iota$, is \emph{intrinsic} to the geometry of the manifold and thus unaffected by the imbedding; in contradistinction, the reach, $\tau$, is \emph{extrinsic} and thus can be made smaller by changing the imbedding.  In particular, when the obstruction to the reach being large is a ``bottleneck'' in the sense that the manifold is imbedded in such a way as to place distant neighborhoods of the manifold close together in Euclidean distance (see Figure \ref{fig3} for an example), we may expect $\tau \ll \iota$.  Intuitively, this matches the notion that the geometry of the data would be simple if we were to have access to the ``correct'' coordinate system and that the difficulty in understanding the geometry comes from its imbedding in the ambient space.
\begin{figure}[H]
    \centering
    \begin{subfigure}[b]{0.4\textwidth}
        \centering
        \includegraphics[width=\textwidth]{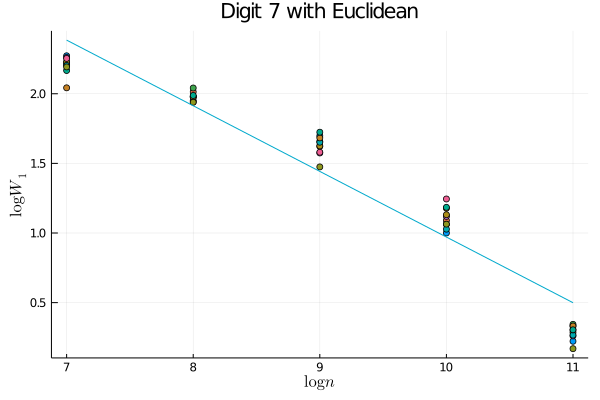}
        \label{fig:eucliddecay}
    \end{subfigure}
    \hfill
    \begin{subfigure}[b]{0.4\textwidth}
        \centering
        \includegraphics[width=\textwidth]{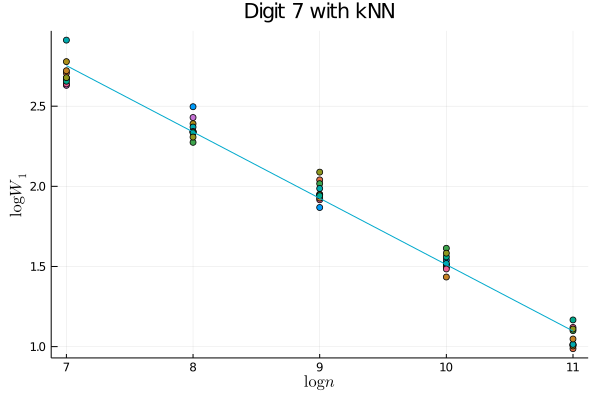}
        \label{fig:mflddecay}
    \end{subfigure}
    
    \begin{subfigure}[b]{0.4\textwidth}
        \centering
        \includegraphics[width=\textwidth]{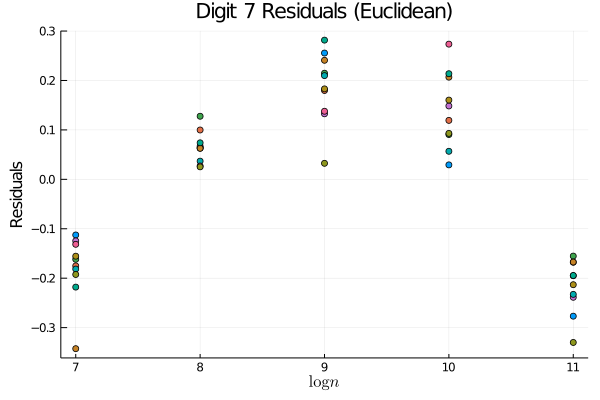}
        \label{fig:euclidresid}
    \end{subfigure}
    \hfill
    \begin{subfigure}[b]{0.4\textwidth}
        \centering
        \includegraphics[width=\textwidth]{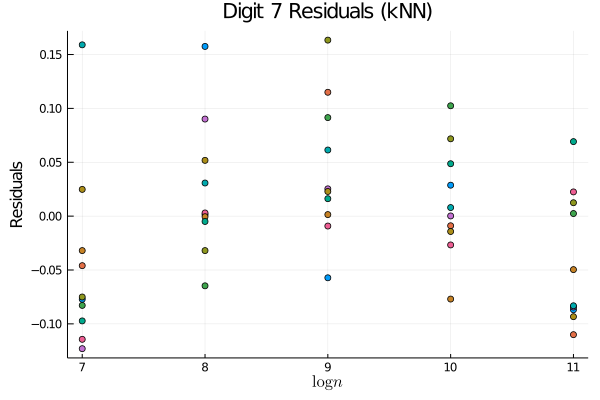}
        \label{fig:mfldresid}
    \end{subfigure}
       \caption{Two $\log$-$\log$ plots of comparing how $W_1(P_n, P_n')$ decays to how $W_1^M(P_n, P_n')$ decays as $n$ gets larger, as well as the residuals from a linear fit.  The data are images of the digit $7$ from MNIST with Wasserstein distances computed with the Sinkhorn algorithm \citep{cuturi2013}.  The manifold distances are approximated by a $k$-NN graph, as described in Section \ref{sec:dimestim}.}
       \label{fig1}
\end{figure}
We emphasize that, like many estimators of intrinsic dimension, we do not claim robustness to off-manifold noise \citep{levina2004maximum,farahmand2007manifold,kim2019minimax}.  Indeed, any ``fattening'' of the manifold will force any consistent estimator of intrinsic dimension to asymptotically grow to the full, ambient dimension as the number of samples grows.  Various works have included off-manifold noise in different ways, often with the assumption that either the noise is known \citep{koltchinskii2000empirical} or the manifold is linear \citep{niles2019estimation}.  Methods that do not make these simplifying assumptions are often highly sensitive to scaling parameters that are required inputs in such methods as multi-scale, local SVD \citep{little2009multiscale}.  Extensions of our method to such noisy settings are a promising avenue of future research, particularly in understanding the effect of this noise on downstream applications as is done for Lipschitz classification in metric spaces and the resulting dimension-distortion tradeoff found in \cite{gottlieb2016adaptive}; in this work, however, we confine our theoretical study to the noiseless setting.  The primary theoretical advantage of our estimator over that of \cite{levina2004maximum,farahmand2007manifold} is that we do not require the stringent regularity assumptions for our nonasymptotic rates to hold.  We leave it for future empirical works whether this weakening of assumptions allows for a better practical estimator on real-world data sets.

Our main contributions are as follows:
\begin{itemize}
    
    \item In Section \ref{sec:dimestim}, we introduce a new estimator of intrinsic dimension.  In Theorem \ref{thm:estimator} we prove non-asymptotic bounds on the quality of the introduced estimator.  Moreover, unlike the MLE estimator of \cite{levina2004maximum} with non-asymptotic analysis in \cite{farahmand2007manifold}, minimal regularity of the density of the population distribution is required for our guarantees and, unlike that suggested in \cite{kim2019minimax}, our estimator is both computationally efficient and has sample complexity independent of the ambient dimension.
    
    \item In the course of proving Theorem \ref{thm:estimator}, we adapt the techniques of \cite{Bernstein00} to provide new, non-asymptotic bounds on the quality of kNN distance as an estimate of intrinsic distance in Proposition \ref{prop:knn}, with explicit sample complexity in terms of the reach of the underlying space.  To our knowledge, these are the first such non-asymptotic bounds.
\end{itemize}
We further note that the techniques we develop to prove the non-asymptotic bounds on our dimension estimator also serve to provide new statistical rates in learning Generative Adversarial Networks (GANs) with a H{\"o}lder discriminator class:
\begin{itemize}
    \item We prove in Theorem \ref{thm:gans} that if $\widehat{\mu}$ is a H{\"o}lder GAN, then the distance between $\widehat{\mu}$ and $\pp$, as measured by the H{\"o}lder IPM, is governed by rates dependent only on the intrinsic dimension of the data, independent of the ambient dimension or the dimension of the feature space.  In particular, we prove in great generality that if $\pp$ has intrinsic dimension $d$, then the rate of a Wasserstein GAN is $n^{- \frac 1d}$.  This improves on the recent work of \cite{schreuder2020statistical}.
\end{itemize}

The work is presented in the order of the above listed contributions, preceded by a brief section on the geometric preliminaries and prerequisite results.  We conclude the introduction by fixing notation and surveying some related work.

\paragraph{Notation:} We fix the following notation.  We always let $\pp$ be a probability distribution on $\rr^D$ and, whenever defined, we let $d = \dim \supp \pp$.  We reserve $X_1, \dots, X_n$ for samples taken from $\pp$ and we denote by $P_n$ their empirical distribution.  We reserve $\beta$ for the smoothness of a H{\"o}lder class, $\Omega \subset \rr^D$ is always a bounded open domain, and $\Delta$ is always the intrinsic diameter of a closed set.  We also reserve $M$ for a compact manifold.  In general, we denote by $\mathcal{S}$ the support of a distribution $\pp$ and we reuse $M = \supp \pp$ if we restrict ourselves to the case where $\mathcal{S} = M$ is a compact manifold, with Riemannian metric induced by the Euclidean metric.  We denote by $\vol M$ the volume of the manifold with respect to its inherited metric and we reserve $\omega_d$ for the volume of the unit ball in $\rr^d$.  When a compact manifold manifold $M$ can be assumed from context, we take the \emph{uniform} measure on $M$ to be the volume measure of $M$ normalized so that $M$ has unit measure.

\subsection{Related Work}

\paragraph{Dimension Estimation} There is a long history of dimension estimation, beginning with linear methods such as thresholding principal components \citep{fukunaga1971algorithm}, regressing k-Nearest-Neighbors (kNN) distances \citep{pettis1979intrinsic}, estimating packing numbers \citep{kegl2002intrinsic,grassberger2004measuring,camastra2002estimating}, an estimator based solely on neighborhood (but not metric) information that was recently proven consistent \citep{kleindessner2015dimensionality}, and many others.  An exhaustive recent survey on the history of these techniques can be found in \cite{camastra2016intrinsic}.  Perhaps the most popular choice among current practitioners is the MLE estimator of \cite{levina2004maximum}.

The MLE estimator is constructed as the maximum likelihood of a parameterized Poisson process.  As worked out in \cite{levina2004maximum}, a local estimate of dimension for $k \geq 2$ and $x \in \rr^D$ is given by
\begin{equation*}
    \widehat{m}_k(x) = \left(\frac 1{k-1} \sum_{j = 1}^k \log \frac{T_k(x)}{T_j(x)}\right)^{-1}
\end{equation*}
where $T_j(x)$ is the distance between $x$ and its $j^{th}$ nearest neighbor in the data set.  The final estimate for fixed $k$ is given by averaging $\widehat{m}_k$ over the data points in order to reduce variance.  While not included in the original paper, a similar motivation for such an estimator could be noting that if $X$ is uniformly distributed on a ball of radius $R$ in $\rr^{d}$, then $\ee\left[\log \frac R{\norm{X}}\right] = \frac 1d$; the local estimator $\widehat{m}_k(x)$ is the empirical version under the assumption that the density is smooth enough to be approximately constant on this small ball.  The easy computation is included for the sake of completeness in Appendix \ref{app:miscellany}.  In \cite{farahmand2007manifold}, the authors examined a closely related estimator and provided non-asymptotic guarantees with an exponential dependence on the intrinsic dimension, albeit with stringent regularity conditions on the density.

In addition to the estimators motivated by the volume growth of local balls discussed in the previous paragraph, \cite{kim2019minimax} proposed and analyzed a dimension estimator based on Travelling Salesman Paths (TSP).  One major advantage to the TSP estimator is the lack of necessary regularity conditions on the density, requiring only an upper bound of the likelihood of the population density with respect to the volume measure on the manifold.  On the other hand, the upper bound on sample complexity that that paper presents depends exponentially on the ambient dimension, which is pessimistic when the intrinsic dimension is substantially smaller.  In addition, it is unclear how practical the estimator is due to the necessity of computing a solution to TSP; even ignoring this issue, \cite{kim2019minimax} note that practical tuning of the constants involved in their estimator is difficult and thus deploying their estimator as is on real-world datasets is unlikely.

\paragraph{Manifold Learning}  The notion of reach was first introduced in \cite{federer1959curvature}, and subsequently used in the machine learning and computational geometry communities in such works as \cite{niyogi2008finding,niyogi2011topological,aamari2019estimating,amenta1999surface,fefferman2016testing,fefferman2018fitting,narayanan2010sample,efimov2019adaptive,boissonnat2019reach}.  Perhaps most relevant to our work, \cite{narayanan2010sample,fefferman2016testing} consider the problem of testing membership in a class of manifolds of large reach and derive tight bounds on the sample complexity of this question.  Our work does not fall into the purview of their conclusions as we assume that the geometry of the underlying manifold is nice and estimate the intrinsic dimension.  In the course of proving bounds on our dimension estimator, we must estimate the intrinsic metric of the data.  We adapt the proofs of \cite{tenenbaum2000,Bernstein00,niyogi2008finding} and provide tight bounds on the quality of a $k$-Nearest Neighbors ($k$NN) approximation of the intrinsic distance.

\paragraph{Statistical Rates of GANs}  
Since the introduction of Generative Adversarial Networks (GANs) in \cite{goodfellow2014generative}, there has been a plethora of empirical improvements and theoretical analyses.  Recall that the basic GAN problem selects an estimated distribution $\widehat{\mu}$ from a class of distributions $\mathcal{P}$ minimizing some adversarially learned distance between $\widehat{\mu}$ and the empirical distribution $P_n$.  Theoretical analyses aim to control the distance between the learned distribution $\widehat{\mu}$ and the population distribution $\pp$ from which the data comprising $P_n$ are sampled.  In particular statistical rates for a number of interesting discriminator classes have been proven including Besov balls \citep{uppal2019nonparametric}, balls in an RKHS \citep{liang2018HowWell}, and neural network classes \citep{chen2020statistical} among others.  The latter paper, \cite{chen2020statistical} also considers GANs where the discriminative class is a H{\"o}lder ball, which includes the popular Wasserstein GAN framework of \cite{arjovsky2017wasserstein}.  They show that if $\widehat{\mu}$ is the empirical minimizer of the GAN loss and the population distribution $\pp \ll \mathsf{Leb}_{\rr^D}$ then
\begin{equation*}
    \ee\left[d_\beta(\widehat{\mu}, \pp)\right] ~\lesssim~ n^{- \frac{\beta}{2 \beta + D}}
\end{equation*}
up to factors polynomial in $\log n$.  Thus, in order to beat the curse of dimensionality, one requires $\beta = \Omega(D)$; note that the larger $\beta$ is, the weaker the IPM is as the H{\"o}lder ball becomes smaller.  In order to mitigate this slow rate, \cite{schreuder2020statistical} assume that both $\mathcal{P}$ and $\pp$ are distributions arising from Lipschitz pushforwards of the uniform distribution on a $d$-dimensional hypercube; in this setting, they are able to remove dependence on $D$ and show that
\begin{equation*}
    \ee\left[ d_\beta(\widehat{\mu}, \pp)\right] ~\lesssim~ L n^{- \frac{\beta}{d}} \vee n^{- \frac 12}.
\end{equation*}
This last result beats the curse of dimensionality, but pays with restrictive assumptions on the generative model as well as dependence on the Lipschitz constant of the pushforward map.  More importantly, the result depends exponentially not on the intrinsic dimension of $\pp$ but rather on the dimension of the feature space used to represent $\pp$.  In practice, state-of-the-art GANs used to produce images often choose $d$ to be on the order of $128$, which is much too large for the \cite{schreuder2020statistical} result to guarantee good performance.


\section{Preliminaries}\label{sec:covering}
\subsection{Geometry}

In this work, we are primarily concerned with the case of compact manifolds isometrically imbedded in some large ambient space, $\rr^D$.  We note that this focus is largely in order to maintain simplicity of notation and exposition; extensions to more complicated, less regular sets with intrinsic dimension defined as the Minkowski dimension can easily be attained with our techniques.  The key example to keep in mind is that of image data, where each pixel corresponds to a dimension in the ambient space, but, in reality, the distribution lives on a much smaller, imbedded subspace.  Many of our results can be easily extended to the non-compact case with additional assumptions on the geometry of the space and tails of the distribution of interest.

Central to our study is the analysis of how complex the support of a distribution is.  We measure complexity of a metric space by its entropy:
\begin{definition}
    Let $(X, d)$ be a metric space.  The covering number at scale $\epsilon > 0$, $N(X, d, \epsilon)$, is the minimal number $s$ such that there exist points $x_1, \dots, x_s$ such that $X$ is contained in the union of balls of radius $\epsilon$ centred at the $x_i$.  The packing number at scale $\epsilon > 0$, $D(X, d, \epsilon)$, is the maximal number $s$ such that there exist points $x_1, \dots, x_s \in X$ such that $d(x_i, x_j) > \epsilon$ for all $i \neq j$.  The entropy is defined as $\log N(X, d, \epsilon)$.
\end{definition}
We recall the classical packing-covering duality, proved, for example, in \cite[Lemma 5.12]{vanHandel2014}:
\begin{lemma}\label{lem:duality}
  For any metric space $X$ and scale $\epsilon > 0$,
  \begin{equation*}
      D(X, d, 2\epsilon) \leq N(X, d, \epsilon) \leq D(X, d, \epsilon).
  \end{equation*}
\end{lemma}
The most important geometric quantity that determines the complexity of a problem is the dimension of the support of the population distribution.  There are many, often equivalent ways to define this quantity in general.  One possibility, introduced in \cite{assouad1983plongements} and subsequently used in \cite{dasgupta2008random,kpotufe2012tree,kpotufe2013adaptivity} is that of doubling dimension:
\begin{definition}
    Let $\cs \subset \rr^D$ be a closed set.  For $x \in \cs$, the doubling dimension at $x$ is the smallest $d$ such that for all $r > 0$, the set $B_r(x) \cap \cs$ can be covered by $2^d$ balls of radius $\frac r2$, where $B_r(x)$ denotes the Euclidean ball of radius $r$ centred at $x$.  The doubling dimension of $\cs$ is the supremum of the doubling dimension at $x$ for all $x \in \cs$.
\end{definition}
This notion of dimension plays well with the entropy, as demonstrated by the following \cite[Lemma 6]{kpotufe2012tree}:
\begin{lemma}[\citep{kpotufe2012tree}] \label{lem:coveringdim}
  Let $\cs$ have doubling dimension $d$ and diameter $\Delta$.  Then $N(\cs, \epsilon) \leq \left(\frac \Delta\epsilon\right)^d$.
\end{lemma}
We remark that a similar notion of dimension is that of the \emph{Minkowski dimension}, which is defined as the asymptotic rate of growth of the entropy as the scale tends to zero.  Recently, \cite{nakada2020adaptive} examined the effect that an assumption of small Minkowski dimension has on learning with neural networks; their central statistical result can be recovered as an immediate consequence of our complexity bounds below.

In order to develop non-asymptotic bounds, we need some understanding of the geometry of the support, $M$.  We first recall the definition of the geodesic distance:
\begin{definition}
    Let $\cs \subset \rr^D$ be closed.  A piecewise smooth curve in $\cs$, $\gamma$, is a continuous function $\gamma: I \to \cs$, where $I \subset \rr$ is an interval, such that there exists a partition $I_1, \cdots, I_J$ of $I$ such that $\gamma_{I_j}$ is smooth as a function to $\rr^D$.  The length of $\gamma$ is induced by the imbedding of $\cs \subset \rr^D$.  For points $p, q \in \cs$, the intrinsic (or geodesic) distance is
    \begin{equation*}
        d_\cs(p, q) = \inf\left\{ \mathsf{length}\,(\gamma)| \gamma(0) = p \text{ and } \gamma(1) = q \text{ and } \gamma \text{ is a piecewise smooth curve in } \cs\right\}.
    \end{equation*}
\end{definition}
It is clear from the fact that straight lines are geodesics in $\rr^D$ that for any points $p, q \in \cs$, $\norm{p - q} \leq d_\cs(p,q)$.  We are concerned with two relevant geometric quantities, one extrinsic and the other intrinsic.
\begin{definition}
    Let $\cs \subset \rr^D$ be a closed set.  Let the medial axis $\med(\cs)$ be defined as
    \begin{equation*}
        \med(\cs) = \left\{x \in \rr^D | \text{ there exist } p \neq q \in \cs \text{ such that } \norm{p - x} = \norm{q - x} = d(x, \cs) \right\}.
    \end{equation*}
    In other words, the medial axis is the set of points in $\rr^D$ that have at  least two projections to $\cs$.  Define the reach, $\tau_\cs$ of $\cs$ as $d(\cs, \med(\cs))$, the minimal distance between a set and its medial axis.
    
    If $\cs = M$ is a compact manifold with the induced Euclidean metric, we define the injectivity radius $\iota = \iota_M$ as the maximal $r$ such that if $p, q \in M$ such that $d_M(p, q) < r$ then there exists a unique length-minimizing geodesic connecting $p$ to $q$ in $M$.
\end{definition}
For more detail on the injectivity radius, see \cite{lee2013smooth}, especially Chapters 6 and 10.  The difference between $\iota_M$ and $\tau_M$ is in the choice of metric with which we equip $M$.  We could choose to equip $M$ with the metric induced by the Euclidean distance $\norm{\cdot}$ or we could choose to use the intrinsic metric $d_M$ defined above.  The reach quantifies the maximal radius of a ball with respect to the \emph{Euclidean} distance such that the intersection of this ball with $M$ behaves roughly like Euclidean space.  The injectivity radius, meanwhile, quantifies the maximal radius of a ball with respect to the \emph{intrinsic} distance such that this ball looks like Euclidean space.  While neither quantity is necessary for our dimension estimator, both figure heavily in the analysis.  The final relevant geometric quantity is the sectional curvature.  The sectional curvature of $M$ at a point $p \in M$ given two directions tangent to $M$ at $p$ is given by the Gaussian curvature at $p$ of the image of the exponential map applied to a small neighborhood of the origin in the plane determined by the two directions.  Intuitively, the sectional curvature measures how tightly wound the manifold is locally around each point.  For an excellent exposition on the topic, see \cite[Chapter 8]{lee2013smooth}.

We now specialize to consider compact, dimension $d$ manifolds $M$ imbedded in $\rr^D$ with the induced metric (see \cite{lee2013smooth} for an accessible introduction to the geometric notions discussed here).  One measure of size of the manifold $M$ is the diameter, $\Delta$, with respect to the intrinsic distance defined above.  Another notion of size is the volume measure, $\vol_M$.  This measure can be defined intrinsically as integration with respect to the volume form, where the volume form can be thought of as the analogue of the Lebesgue differential in standard Euclidean space; for more details see \cite{lee2013smooth}.  In our setting, we could equivalently define the volume as the $d$-dimensional Hausdorff measure as in \cite{aamari2019estimating}.  Either way, when we refer to a measure $\mu_M$ that is uniform on the manifold, we consider the normalization such that $\mu_M(M) = 1$, i.e., $\mu_M(\cdot) = \vol_M(\cdot) / \vol(M)$.  

With the brief digression into volume concluded, we return to the notion of the reach, which encodes a number of local and global geometric properties.  We summarize several of these in the following proposition:
\begin{proposition}\label{prop:mfldgeometry}
 Let $M \subset \rr^D$ be a compact manifold isometrically imbedded in $\rr^D$.  Suppose that $\tau = \tau_M > 0$.  The following hold:
 \begin{enumerate}[(a)]
     \item \cite[Proposition 6.1]{niyogi2008finding}] The norm of the second fundamental form of $M$ is bounded by $\frac 1{\tau}$ at all points $p \in M$.
     
     \item \cite[Proposition A.1 (ii)]{aamari2019estimating} The injectivity radius of $M$ is at least $\pi \tau$.
     
     \item \cite[Lemma 3]{boissonnat2019reach} If $p, q \in M$ such that $\norm{p - q} \leq 2 \tau$ then $d_M(p, q) \leq 2 \tau \arcsin\left(\frac{\norm{p - q}}{2 \tau}\right)$.
     
 \end{enumerate}
\end{proposition}
A few remarks are in order.  First, note that the Hopf-Rinow Theorem \citep{hopf1931begriff} guarantees that $M$ is complete, which is fortuitous as completeness is a necessary, technical requirement for several of our arguments.  Second, we note that (c) from Proposition \ref{prop:mfldgeometry} has a simple geometric interpretation: the upper bound on the right hand side is the length of the arc of a circle of radius $\tau$ containing points $p,q$; thus, the maximal distortion of the intrinsic metric with respect to the ambient metric is bounded by the circle of radius $\tau$.

Point (a) in the above proposition demonstrates that control of the reach leads to control of local distortion.  From the definition, it is obvious that the reach provides an upper bound for the size of the global notion of a ``bottleneck," i.e., two points $p, q \in M$ such that $\norm{p - q} = 2 \tau < d_M(p, q)$.  Interestingly, these two local and global notions of distortion are the only ways that the reach of a manifold can be small, as \cite[Theorem 3.4]{aamari2019estimating} tells us that if the reach of a manifold $M$ is $\tau$, then either there exists a bottleneck of size $2 \tau$ or the norm of the second fundamental form is $\frac 1\tau$ at some point.  Thus, in some sense, the reach is the ``correct" measure of distortion.  Note that while (b) above tells us that $\iota_M \gtrsim \tau_M$, there is no comparable upper bound.  To see this, consider Figure \ref{fig3}, which depicts a one-dimensional manifold imbedded in $\rr^2$.  Note that the bottleneck in the center ensures that the reach of this manifold is very small; on the other hand, it is easy to see that the injectivity radius is given by half the length of the entire curve.  As the curve can be extended arbitrarily, the reach can be arbitrarily small relative to the injectivity radius.
\begin{figure}
    \centering
    \includegraphics[width=.5\textwidth]{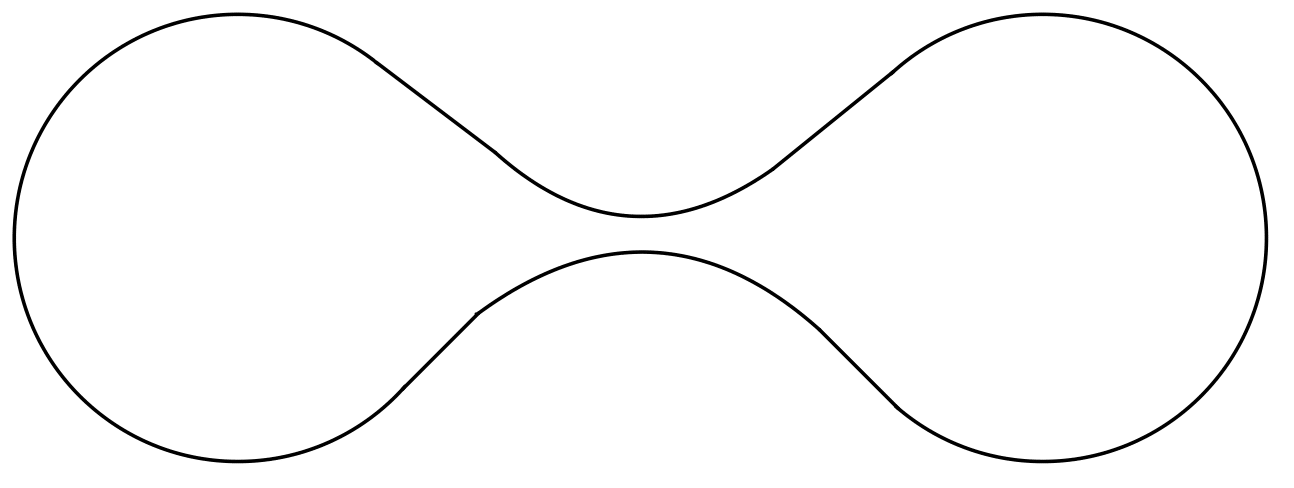}
    \caption{Curve in $\rr^2$ where $\tau \ll \iota$.}
    \label{fig3}
\end{figure}

We now proceed to bound the covering number of a compact manifold using the dimension and the injectivity radius.  We note that upper bounds on the covering number with respect to the ambient metric were provided in \cite{niyogi2008finding,narayanan2010sample}.  A similar bound with less explicit constants can be found in \cite[Lemma 4]{kim2019minimax}.
\begin{proposition}\label{prop:mfldcovering}
 Let $M \subset \rr^D$ be an isometrically imbedded, compact, $d$-dimensional submanifold with injectivity radius $\iota > 0$ such that the sectional curvatures are bounded above by $\kappa_1 \geq 0$ and below by $\kappa_2 \leq 0$.  If $\epsilon < \frac{\pi}{2 \sqrt{k_1}} \wedge \iota$ then
 \begin{equation*}
     N(M, d_M, \epsilon) \leq \frac{\vol M}{\omega_d} d \left(\frac \pi 2\right)^d \epsilon^{-d}.
 \end{equation*}
 If $\epsilon < \frac{1}{\sqrt{-\kappa_2}} \wedge \iota$ then
 \begin{equation*}
     \frac{\vol M}{\omega_d} d 8^{-d}  \epsilon^{-d} \leq D(M, d_M, 2\epsilon).
 \end{equation*}
 Moreover, for all $\epsilon < \iota$,
 \begin{equation*}
     \frac{\vol M}{\omega_d} d \iota^d (- \kappa_2)^{\frac d2} e^{- d \iota \sqrt{- \kappa_2}} \epsilon^{-d} \leq D(M, d_M, \epsilon).
 \end{equation*}
 Thus, if $\epsilon < \tau$, where $\tau$ is the reach of $M$, then
 \begin{equation*}
     \frac{\vol M}{\omega_d} d 8^{-d}  \epsilon^{-d} \leq D(M, d_M, 2\epsilon) \leq N(M, d_M, \epsilon) \leq \frac{\vol M}{\omega_d} d \left(\frac \pi 2\right)^d \epsilon^{-d}.
 \end{equation*}
\end{proposition}
The proof of Proposition \ref{prop:mfldcovering} can be found in Appendix \ref{app:coverings} and relies on the Bishop-Gromov comparison theorem to leverage the curvature bounds from Proposition \ref{prop:mfldgeometry} into volume estimates for small intrinsic balls, a similar technique as found in \cite{niyogi2008finding,narayanan2010sample}.  The key point to note is that we have both upper and lower bounds for $\epsilon < \iota$, as opposed to just the upper bound guaranteed by Lemma \ref{lem:coveringdim}.  As a corollary, we are also able to derive bounds for the covering number with respect to the ambient metric:
\begin{corollary}\label{cor:ambientcovering}
  Let $M$ be as in Proposition \ref{prop:mfldcovering}.  For $\epsilon < \tau$, we can control the covering numbers of $M$ with respect to the Euclidean metric as
  \begin{equation*}
      \frac{\vol M}{\omega_d} d 16^{-d} \epsilon^{- d} \leq D(M, \norm{\cdot}, 2 \epsilon) \leq N(M, \norm{\cdot}, \epsilon) \leq \frac{\vol M}{\omega_d} \left( \frac \pi 2 \right)^d \epsilon^{- d}.
  \end{equation*}
\end{corollary}
The proof of Corollary \ref{cor:ambientcovering} follows from Proposition \ref{prop:mfldcovering} and the metric comparisons for small scales in Proposition \ref{prop:mfldgeometry}; details can be found in Appendix \ref{app:coverings}.

\subsection{H{\"o}lder Classes and their Complexity}
In this section we make the elementary observation that complex function classes restricted to simple subsets can be much smaller than the original class.  While such intuition has certainly appeared before, especially in designing esimators that can adapt to local intrinsic dimension, such as \cite{bickel2007local,kpotufe2012tree,kpotufe2011k,kpotufe2013adaptivity,dasgupta2008random,steinwart2009optimal,nakada2020adaptive}, we codify this approach below.

To illustrate the above phenomenon at the level of empirical processes, we focus on H{\"o}lder functions in $\rr^D$ for some large $D$ and let the ``simple" subset be a subspace of dimension $d$ where $d \ll D$.  We first recall the definition of a H{\"o}lder class:
\begin{definition}
    For an open domain $\Omega \subset \rr^d$ and a function $f: \Omega \to \rr$, define the $\beta$-H{\"o}lder norm as
	\begin{equation*}
		\norm{f}_{C^\beta(\Omega)} = \max_{0 \leq \abs \gamma \leq \abs{\alpha}} \sup_{x \in \Omega} \abs{D^\gamma f(x)} \vee \sup_{x, y \in \Omega} \frac{\abs{D^{\lfloor \beta \rfloor} f(x) - D^{\lfloor \beta \rfloor} f(y)}}{\norm{x - y}^{\beta - \lfloor \beta \rfloor}}.
	\end{equation*}
	Define the H{\"o}lder ball of radius $B$, denoted by $C_B^\beta(\Omega)$, as the set of functions $f: \Omega \to \rr$ such that $\norm{f}_{C^\beta(\Omega)} \leq B$.  If $(M, g)$ is a Riemannian manifold of class $C^{\lfloor \beta \rfloor + 1}$ (see \cite{lee2013smooth}), and $f: M \to \rr$ we define the H{\"o}lder norm analogously, replacing $\abs{D^\gamma f(x)}$ with $\norm{\nabla^\gamma f(x)}_g$, where $\nabla$ is the covariant derivative.
\end{definition}

It is a classical result of \cite{Tikhomirov1993} that, for a bounded, open domain $\Omega \subset \rr^D$, the entropy of a H{\"o}lder ball scales as
\begin{equation*}
    \log N\left(C_B^\beta(\Omega), \norm{\cdot}_\infty, \epsilon\right) \asymp \left(\frac{B}{\epsilon}\right)^{\frac D\beta}
\end{equation*}
as $\epsilon \downarrow 0$.    As a consequence, we arrive at the following result, whose proof can be found in Appendix \ref{app:coverings} for the sake of completeness.

\begin{proposition}\label{prop:keyHolderprop}
  Let $\cs \subset \Omega \subset \rr^d$ be a path-connected closed set contained in an open domain $\Omega$.  Let $\widetilde{\F} = C_B^\beta(\Omega)$ and let $\F = \widetilde{\F}|_\cs$.  Then,
  \begin{equation*}
      D\left(\cs, \left(\frac{\epsilon}{B}\right)^{\frac 1\beta}\right) \leq \log D(\F, \norm{\cdot}_\infty, 2\epsilon)\leq \log N(\F, \norm{\cdot}_\infty, \epsilon) \leq 3\beta^2 \log\left(\frac{2 B}{\epsilon}\right) N\left(\cs, \left(\frac{\epsilon}{2 B}\right)^{\frac 1\beta}\right).
  \end{equation*}
\end{proposition}
Note that the content of the above result is really that of \cite{Tikhomirov1993}, coupled with the fact that restriction from $\rr^d$ to $M$ preserves smoothness.

If we apply the easily proven volumetric bounds on covering and packing numbers for $\mathcal{S}$ a Euclidean ball to Proposition \ref{prop:keyHolderprop}, we recover the classical result of \cite{Tikhomirov1993}.  The key insight is that low-dimensional subsets can have covering numbers much smaller than those of a high-dimensional Euclidean ball: if the ``dimension" of $\cs$ is $d$, then we expect the covering number of $\cs$ to scale like $\epsilon^{-d}$.  Plugging this into Proposition \ref{prop:keyHolderprop} tells us that the entropy of $\F$, up to a factor logarithmic in $\frac 1\epsilon$, scales like $\epsilon^{- \frac d \beta} \ll \epsilon^{- \frac D\beta}$.  An immediate corollary of Lemma \ref{lem:coveringdim} and Proposition \ref{prop:keyHolderprop} is:
\begin{corollary}\label{cor:Holderentropy}
  Let $\cs \subset \rr^D$ be a closed set of diameter $\Delta$ and doubling dimension $d$.  Let $\cs \subset \Omega$ open and $\F$ be the restriction of $C_B^\beta(\Omega)$ to $\cs$.  Then
  \begin{equation*}
      \log N(\F, \norm{\cdot}_\infty, \epsilon) \leq 3\beta^2  \left(\frac{2 B \Delta^\beta}{\epsilon}\right)^{\frac d\beta}\log\left(\frac{2 B}{\epsilon}\right).
  \end{equation*}
\end{corollary}
\begin{proof}
  Combine the upper bound in Proposition \ref{prop:keyHolderprop} with the bound in Lemma \ref{lem:coveringdim}.
\end{proof}
The conclusion of Corollary \ref{cor:Holderentropy} is very useful for upper bounds as it tells us that the entropy for H{\"o}lder balls scales at most like $\epsilon^{- \frac d\beta}$ as $\epsilon \downarrow 0$.  If we desire comparable lower bounds, we require some of the geometry discussed above.  Combining Proposition \ref{prop:keyHolderprop} and Corollary \ref{cor:ambientcovering} yields the following bound:
\begin{corollary}\label{cor:mfldHolderentropy}
  Let $M \subset \rr^D$ be an isometrically imbedded, compact submanifold with reach $\tau > 0$ and let $\epsilon \leq \tau$.  Suppose $\Omega \supset M$ is an open set and let $\F'$ be the restriction of $C_B^{\beta}(\Omega)$ to $M$.  Then for $\epsilon \leq \tau$,
  \begin{equation*}
      \frac{\vol M}{\omega_d} d 16^{-d} \left(\frac{2 B}{\epsilon}\right)^{\frac d\beta}\leq \log D(\F', \norm{\cdot}_\infty, 2 \epsilon) \leq \log N(\F', \norm{\cdot}_\infty, \epsilon) \leq 3\beta^2 \log\left(\frac{2 B}{\epsilon}\right) \frac{\vol M}{\omega_d} d \left(\frac \pi 2\right)^d \left(\frac{2 B}{\epsilon}\right)^{\frac d \beta}.
  \end{equation*}
  If we set $\mathcal{F} = C_B^{\beta}(M)$, then we have that for all $\epsilon < \iota$,
  \begin{equation*}
      \frac{\vol M}{\omega_d} d \iota^d (- \kappa_2)^{\frac d2} e^{- d \iota \sqrt{- \kappa_2}}  \epsilon^{- \frac d\beta} \leq \log N(\F, \norm{\cdot}_\infty, \epsilon) \leq 3 \beta^2 \log\left( \frac{2B}\epsilon\right) \frac{\vol M}{\omega_d} d \left(\frac \pi 2\right)^d \epsilon^{- \frac d\beta}.
  \end{equation*}
\end{corollary}
In essence, Corollary \ref{cor:mfldHolderentropy} tells us that the rate of $\epsilon^{- \frac d\beta}$ for the growth of the entropy of H{\"o}lder balls is sharp for sufficiently small $\epsilon$.  The key difference between the first and second statements is that the first is with respect to an ambient class of functions while the second is with respect to an intrinsic class.  To better illustrate the difference, consider the case where $\beta = B = 1$, i.e., the class of Lipschitz functions on the manifold.  In both cases, asymptotically, the entropy of Lipschitz functions scales like $\epsilon^{- d}$; if we restrict to functions that are Lipschitz with respect to the ambient metric, then the above bound only applies for $\epsilon < \tau$; on the other hand, if we consider the larger class of functions that are Lipschitz with respect to the intrinsic metric, the bound applies for $\epsilon < \iota$.  In the case where $\iota \gg \tau$, this can be a major improvement.

The observations in this section are undeniably simple; the real interest comes in the diverse applications of the general principle, some of which we detail below.  As a final note, we remark that our guiding principle of simplifying function classes by restricting them to simple sets likely holds in far greater analysis than is explored here; in particular, Sobolev and Besov classes (see, for example, \cite[\S 4.3]{gine2016mathematical}) likely exhibit similar behavior.

%% file: secDimensionEstimation.tex
We outlined the intuition behind our dimension estimation in the introduction.  In this section, we formally define the estimator and analyse its theoretical performance.  We first apply standard empirical process theory and our complexity bounds in the previous section to upper bound the expected H{\"o}lder IPM (defined in \eqref{eq:holderipm}) between empirical and population distributions:
\begin{lemma}\label{lemma:holderipm}
    Let $\mathcal{S} \subset \rr^D$ be a compact set contained in a ball of radius $R$.  Suppose that we draw $n$ independent samples from a probability measure $\pp$ supported on $\mathcal{S}$ and denote by $P_n$ the corresponding empirical distribution.  Let $P_n'$ denote an independent identically distributed measure as $P_n$.  Then we have
    \begin{equation*}
        \ee\left[d_{\beta, B}(P_n, \pp)\right] \leq \ee\left[d_{\beta, B}(P_n, P_n')\right] \leq 16 B \inf_{\delta > 0} \left(2 \delta + \frac{3 \sqrt{6}}{\sqrt{n}} \beta \sqrt{\log \frac 1\delta} \int_\delta^1 \sqrt{N(\mathcal{S}, \norm{\cdot}, \epsilon)} d \epsilon  \right).
    \end{equation*}
    In particular, there exists a universal constant $K$ such that if $N(\mathcal{S}, \norm{\cdot}, \epsilon) \leq C_1 \epsilon^{- d}$ for some $C, d > 0$, then
    \begin{equation*}
        \ee\left[d_\beta(P_n, \pp) \right] \leq C  \beta B \left(1 + \sqrt{\log n} \mathbf{1}_{\{ d = 2 \beta \}} \right) \left(n^{- \frac \beta d} \vee n^{- \frac 12}\right).
    \end{equation*}
    holds with $C = K C_1$.
\end{lemma}
The proof uses the symmetrization and chaining technique and applies the complexity bounds of H{\"o}lder functions found above; the details can be found in Appendix \ref{app:miscellany}.

We now specialize to the case where $\beta = B = 1$, due to the computational tractability of the resulting Wasserstein distance.  Applying Kantorovich-Rubenstein duality \citep{kantorovich1958space}, we see that this special case of Lemma \ref{lemma:holderipm} recovers the special $p = 1$ case of \cite{weed2019sharp}.  From here on, we suppose that $d > 2$ and  our metric on distributions is $d_{1,1} = W_1$.

We begin by noting that if we have $2n$, independent samples from $\pp$, then we can split them into two data sets of size $n$, and denote by $P_n, P_n'$ the empirical distributions thus generated.  We then note that Lemma \ref{lemma:holderipm} implies that if $\supp \pp \subset M$ and $M$ is of dimension $d$, then
\begin{equation*}
    \ee\left[W_1(P_n, P_n')\right] \leq C_{M,d} n^{- \frac 1d}.
\end{equation*}
If we were to establish a lower bound as well as concentration of $W_1(P_n, P_n')$ about its mean, then we could consider the following estimator.  Given a data set of size $2(\alpha + 1)n$, we can break the data into four samples, $P_n, P_n'$ each of size $n$ and $P_{\alpha n}, P_{\alpha n}'$ of size $\alpha n$.  Then we would have
\begin{equation*}\label{eq:dimensionestim}
    d_n := - \frac 1{\log_\alpha\left( \frac{W_1(P_{\alpha n}, P_{\alpha n}')}{W_1(P_n, P_n')}\right)} = \frac{\log \alpha }{\log W_1(P_n, P_n') - \log W_1(P_{\alpha n}, P_{\alpha n}')} \approx d.
\end{equation*}
Which distance on $M$ should be used to compute the Wasserstein distance, the Euclidean metric $\norm{\cdot}$ or the intrinsic metric $d_M(\cdot, \cdot)$?  As can be guessed from Corollary \ref{cor:mfldHolderentropy}, asymptotically, both will work, but for finite sample sizes when $\iota \gg \tau$, the latter is much better.  One problem remains, however: because we are not assuming $M$ to be known, we do not have access to $d_M$ and thus we cannot compute the necessary Wasserstein cost.  In order to get around this obstacle, we recall the graph distance induced by a $k$NN graph:
\begin{definition}
  Let $X_1, \dots, X_n \in \rr^D$ be a data set and fix $\epsilon > 0$.  We let $G(X, \epsilon)$ denote the weighted graph with vertices $X_i$ and edges of weight $\norm{X_i - X_j}$ between all vertices $X_i, X_j$ such that $\norm{X_i - X_j} \leq \epsilon$.  We denote by $d_{G(X, \epsilon)}$ (or $d_G$ if $X, \epsilon$ are clear from context) the geodesic distance on the graph $G(X, \epsilon)$.  We extend this metric to all of $\rr^D$ by letting
  \begin{equation*}
      d_G(p, q) = \norm{p - \pi_G(p)} + d_G(\pi_G(p), \pi_G(q)) + \norm{q - \pi_G(q)}
  \end{equation*}
  where $\pi_G(p) \in \argmin_{X_i} \norm{p - X_i}$.
\end{definition}
We now have two Wasserstein distances, each induced by a different metric; to mitigate confusion, we introduce the following notation:
\begin{definition}
  Let $X_1, \dots, X_n, X_1', \dots, X_n' \in \rr^D$, sampled independently from $\pp$ such that $\supp \pp \subset M$.  Let $P_n, P_n'$ be the empirical distributions associated to the data $X, X'$.  Let $W_1(P_n, P_n')$ denote the Wasserstein cost with respect to the Euclidean metric and  $W_1^M(P_n, P_n')$ denote the Wasserstein cost associated to the manifold metric, as in \eqref{eq:wassersteindef}.  For a fixed $\epsilon > 0$, let $W_1^G(P_n, P_n')$ denote the Wasserstein cost associated to the metric $d_{G(\supp P_n \cup \supp P_n', \epsilon)}$.  Let $d_n$, $\widehat{d}_n$, and $\widetilde{d}_n$ denote the dimension estimators from \eqref{eq:dimensionestim} induced by each of the above metrics.
\end{definition}
Given sample distributions $P_n, P_n'$, we are able to compute $W_1(P_n, P_n')$ and $W_1^G(P_n, P_n')$ for any fixed $\epsilon$, but not $W_1^M(P_n, P_n')$ because we are assuming that the learner does not have access to the manifold $M$.  On the other hand, adapting techniques from \cite{weed2019sharp}, we are able to provide a non-asymptotic lower bound on $W_1(P_n, P_n')$ and $W_1^M(P_n, P_n')$:
\begin{proposition}\label{prop:w1lower}
    Suppose that $\pp$ is a measure on $\rr^D$ such that $\supp \pp = M$, where $M$ is a $d$-dimensional, compact manifold with reach $\tau > 0$ and such that the density of $\pp$ with respect to the uniform measure on $M$ is lower bounded by $w > 0$.  Suppose that
    \begin{equation*}
        n > \frac{ d \vol M}{4 w\omega_d} \left(\frac \tau 8\right)^{-d}.
    \end{equation*}
    Then, almost surely,
    \begin{equation*}
        W_1(P_n, \pp) \geq \frac 1{32}\left(\frac{ d \vol M}{4 w \omega_d}\right)^{\frac 1d}  n^{- \frac 1d}.
    \end{equation*}
    If we assume only that
    \begin{equation*}
        n > \left(\frac{d (- \kappa_2)^{\frac d2} \vol M}{4 w \omega_d} e^{d \iota \sqrt{- \kappa_2}} \right) \iota^{- d}
    \end{equation*}
    then, almost surely,
    \begin{equation*}
        W_1^M (P_n, \pp) \geq \frac 1{32} \left(\frac{ d \vol M}{4 w \omega_d} \right)^{\frac 1d} (- \kappa_2)^{ \frac 12} e^{ \iota \sqrt{- \kappa_2}} n^{- \frac 1 d}.
    \end{equation*}
\end{proposition}
An easy proof, based on the techniques \cite[Proposition 6]{weed2019sharp} can be found in Appendix \ref{app:miscellany}.  Similarly, we can apply the same proof technique as Lemma \ref{lemma:holderipm} to establish the following upper bound:
\begin{proposition}\label{prop:w1}
	 Let $M \subset \rr^D$ be a compact manifold with positive reach $\tau$ and dimension $d > 2$.  Furthermore, suppose that $\pp$ is a probability measure on $\rr^D$ with $\supp \pp \subset M$.  Let $X_1, \dots, X_n, X_1', \dots, X_n' \sim \pp$ be independent with corresponding empirical distributions $P_n, P_n'$.  Then if $\diam M = \Delta$, we have:
	 \begin{equation*}
	 	\ee\left[W_1^M(P_n, \pp)\right] \leq \ee\left[W_1^M(P_n, P_n')\right] \leq C \left(\frac{ \vol M}{n \omega_d}\right)^{\frac 1d} \sqrt[]{\log\left(\frac{n \omega_d \Delta^d}{d \vol_M}\right)}.
	 \end{equation*}
\end{proposition}
The full proof is in Appendix \ref{app:miscellany} and applies symmetrization and chaining, with an upper bound of Corollary \ref{cor:mfldHolderentropy}.  We note, as before, that a similar asymptotic rate is obtained by \cite{weed2019sharp} in a slightly different setting.  

We noted above \eqref{eq:dimensionestim} that we required two facts to make our intuition precise.  We have just shown that the first holds; we turn now to the second: concentration.  To make this rigorous, we need one last technical concept: the $T_2$-inequality.
\begin{definition}
  Let $\mu$ be a measure on a metric space $(M, d)$.  We say that $\mu$ satisfies a $T_2$-inequality with constant $c_2$ if for all measures $\nu \ll \mu$, we have
  \begin{equation*}
      W_2(\mu, \nu) \leq \sqrt{2 c_2 D(\nu || \mu)}
  \end{equation*}
  where $D(\nu || \mu) =\ee_\mu \left[\log \frac{d \nu}{d \mu}\right]$ is the well-known KL-divergence.
\end{definition}
The reason that the $T_2$ inequality is useful for us is that \cite{bobkov1999} tell us that such an inequality implies, and is, by \cite{gozlan2009characterization}, equivalent to Lipschitz concentration.  We note further that $W_1(P_n, P_n')$ is a Lipschitz function of the dataset and thus concentrates about its mean.  The constant in the $T_2$ inequality depends on the measure $\mu$ and upper bounds for specific classes of measures are both well-known and remain an active area of research; for a more complete survey, see \cite{Bakry2014}.  We have the following bound: 
\begin{proposition}\label{prop:t2control}
    Let $\pp$ be a probability measure on $\rr^D$ that has density with respect to the (normalized) volume measure of $M$, lower bounded by $w$ and upper bounded by $W$, where $M$ is a $d$-dimensional manifold with reach $\tau > 0$ and $\diam M = \Delta$.  Then we have:
    \begin{equation}
        c_2 \leq \frac{2\tau^2}{d- 1} \frac W w \exp\left(d \log 3 + \frac{3 d^2 \Delta^2}{\tau^2}\right).
    \end{equation}
\end{proposition}
In order to bound the $T_2$ constant in our case, we rely on the landmark result of \cite{otto2000generalization} that relates $c_2$ to another functional inequality, the log-Sobolev inequality \citep[Chapter 5]{Bakry2014}.  There are many ways to control the log-Sobolev constant in various situations, many of which are covered in \cite{Bakry2014}.  We use results from \cite{wang1997estimation}, which incorporate the intrinsic geometry of the distribution, as our bound.  A detailed proof can be found in Appendix \ref{app:t2control}.  We note that many other estimates with under slightly different conditions exits, such as that in \cite{wang1997logarithmic}, which requires second-order control of the density of the population distribution with respect to the volume measure and the bound in \cite{block2020fast}, which provides control using a measure of nonconvexity.  With added assumptions, we can gain much sharper control over $c_2$; for example, if we assume a positive lower bound on the curvature of the support, we can apply the well-known Bakry-{\'E}mery result \citep{bakry1985diffusions} and get dimension-free bounds.  As another example, if we may assume stronger contol on the curvature of $M$ beyond that guaranteed by the reach, we can remove the exponential dependence on the reach entirely.  For the sake of simplicity and because we already admit an exponential dependence on the intrinsic dimension, we present only the more general bound here.  We now provide a non-asymptotic bound on the quality of the estimator $\widetilde{d}_n$.
\begin{theorem}\label{thm:estimator}
  Let $\pp$ be a probability measure on $\rr^D$ and suppose that $\pp$ has a density with respect to the (normalized) volume measure of $M$ lower bounded by $w$, where $M$ is a $d$-dimensional manifold with reach $\tau > 0$ such that $d \geq 3$ and $\diam M = \Delta$.  Furthermore, suppose that $\pp$ satisfies a $T_2$ inequality with constant $c_2$.  Let $\gamma > 0$ and suppose $\alpha, n$ satisfy
  \begin{align*}
      n & \geq \max\left[\frac{ d \vol M}{4 w \omega_d} \left(\frac{8}{\iota}\right)^d, \left(\frac{8 c_2 }{\Delta^2}\log \frac 1\rho\right)^{\frac{2d}{d - 5}}\right]\\
      \alpha &\geq \max\left[\log^{\frac 2{2\gamma}}\left(\frac{n \omega_d \Delta^d}{d \vol M}\right), (48  w)^{\frac 1\gamma}, 3^{\frac d\gamma}\right] \\
      \alpha n &\geq \frac{ d \vol M}{2 w \omega_d} \left(\frac{16 \pi}{\tau}\right)^d \log\left(\frac{d \vol M}{\rho \omega_d} \left(\frac{16 \pi}{\tau}\right)^d\right).
  \end{align*}
  Suppose we have $2(\alpha + 1) n$ samples drawn independently from $\pp$.  Then, with probability at least $1 - 6\rho$, we have
  \begin{equation*}
      \frac{d}{1 + 3\gamma} \leq \widetilde{d}_n \leq (1 + 3\gamma) d.
  \end{equation*}
  If $\iota$ is replaced by $\tau$ above, we get the same bound with the vanilla estimator $d_n$ replacing $\widetilde{d}_n$.
\end{theorem}
We note that we have not made every effort to minimize the constants in the statement above, with our emphasis being the dependence of these sample complexity bounds on the relevant geometric quantities.  As an immediate consequence of Theorem \ref{thm:estimator}, due to the fact that $d$ is discrete, we can control the probability of error with sufficiently many samples.  We may also apply Proposition \ref{prop:t2control} to replace $c_2$ with our upper bound in terms of the reach.
\begin{corollary}\label{cor:estimator}
    Suppose we are in the situation of Theorem \ref{thm:estimator} and that $\pp$ has density upper bounded by $W$ with respect to the normalized uniform measure on $M$.  Suppose further that $\alpha, n$ satisfy
    \begin{align*}
        n & \geq \max\left[\frac{ d \vol M}{4 w \omega_d} \left(\frac{8}{\iota}\right)^d, \left(8 \frac{2\tau^2}{\Delta^2(d- 1)} \frac W w\exp\left(d \log 3 + \frac{3 d^2 \Delta^2}{\tau^2}\right) \log \frac 1\rho\right)^{\frac{2d}{d - 5}}\right] \\
        \alpha &\geq \max\left[\log^{2 d^2}\left(\frac{n \omega_d \Delta^d}{d \vol M}\right), (48w)^{3d}, 3^{3 d^2}\right] \\
        \alpha n &\geq \frac{d \vol M}{2 w \omega_d} \left(\frac{16 \pi}{\tau}\right)^d \log\left(\frac{d \vol M}{\rho \omega_d} \left(\frac{16 \pi}{\tau}\right)^d\right).
    \end{align*}
    Then if we round $\widetilde{d}_n$ to the nearest integer, and denote the resulting estimator by $d_n'$, we have with probability at least $1 - 6\rho$, $d_n' = d$.  Again, replacing $\iota$ by $\tau$ in the previous display yields the same result with $\widehat{d}_n$ replaced by the vanilla estimator $d_n$.
\end{corollary}
\begin{proof}
    Note that because $d \in \mathbb{N}$, if $\abs{\widetilde{d}_n - d} \leq \frac 12$, then rounding $\widehat{d}_n$ to the nearest integer exactly recovers $d$.  Setting $\gamma < \frac 1{4d}$, and plugging into the result of Theorem \ref{thm:estimator}, along with an application of Proposition \ref{prop:t2control} to bound $c_2$, concludes the proof.
\end{proof}

While the appearance of $\iota$ in Theorem \ref{thm:estimator} and Corollary \ref{cor:estimator} may seem minor, it is critical for any practical estimator.  While $\alpha n = \Omega\left(\tau^{-d}\right)$, we may take $n$ as small as $\Omega\left( \iota^{-d} \right)$.  Thus, using $\widetilde{d}_n$ instead of the naive estimator $d_n$ allows us to leverage the entire data set in estimating the intrinsic distances, even on the small sub-samples.  From the proof, it is clear that we want $\alpha$ to be as large as possible; thus if we have a total of $N$ samples, we wish to make $n$ as small as possible.  If $\iota \gg \tau$ then we can make $n$ much smaller (scaling like $\iota^{-d}$) than if we were to simply use the Euclidean distance.  As a result, on any data set where $\iota \gg \tau$, the sample complexity of $\widetilde{d}_n$ can be much smaller than that of $d_n$.

There are two parts to the proof of Theorem \ref{thm:estimator}: first, we need to establish that our metric $d_G$ approximates $d_M$ with high probability and thus $\widetilde{d}_n \approx \widehat{d}_n$; second, we need to show that $\widehat{d}_n$ is, indeed, a good estimate of $d$.  The second part follows from Propositions \ref{prop:w1} and \ref{prop:w1lower}, and concentration; a detailed proof can be found in Appendix \ref{app:dimestim}.  For the first part of the proof, in order to show that $\widehat{d}_n \approx \widetilde{d}_n$, we demonstrate that $d_M \approx d_G$ in the following result:
\begin{proposition}\label{prop:knn}
    Let $\pp$ be a probability measure on $\rr^D$ and suppose that $\supp \pp = M$, a geodesically convex, compact manifold of dimension $d$ and reach $\tau > 0$.
	Suppose that we sample $X_1, \dots, X_n \sim \pp$ independently.  Let $\lambda \leq \frac 12$ and $G = G(X, \tau \lambda)$.  If for some $\rho < 1$,
	\begin{equation*}
		n \geq w_B\left(\frac{\tau \lambda^2}{8}\right)^{-1} \log \frac{N\left(M, d_M, \frac{\tau \lambda^2}{8}\right)}{\rho}
	\end{equation*}
	where for any $\delta > 0$
	\begin{equation*}
	    w_B(\delta) = \inf_{p \in M} \pp(B_\delta^M(p))
	\end{equation*}
    with $B_\delta^M(p)$ the metric ball around $p$ of radius $\delta$.  Then, with probability at least $1 - \rho$, for all $x, y \in M$,
	\begin{equation*}
		\left(1 - \lambda\right) d_M(x,y) \leq d_G(x,y) \leq (1 + \lambda) d_M(x,y).
	\end{equation*}
\end{proposition}
The proof of Proposition \ref{prop:knn} follows the general outline of \cite{Bernstein00}, but is modified in two key ways: first, we control relevant geometric quantities by $\tau$ instead of by the quantities in \cite{Bernstein00}; second, we provide a quantitative, nonasymptotic bound on the number of samples needed to get a good approximation with high probability.  The details are deferred to Appendix \ref{app:metricestim}.

This result may be of interest in its own right as it provides a non-asymptotic version of the results from \cite{tenenbaum2000,Bernstein00}.  In particular, if we suppose that $\pp$ has a density with respect to the uniform measure on $M$ and this density is bounded below by a constant $w > 0$, then Proposition \ref{prop:knn} combined with Proposition \ref{prop:mfldcovering} tells us that if we have
\begin{equation*}
    n \gtrsim  \frac{\vol M}w \left(\tau \lambda^2\right)^{-d} \log \left(\frac{\vol M}{\rho \tau \lambda^2}\right)
\end{equation*}
samples, then we can recover the intrinsic distance of $M$ with distortion $\lambda$.  We further note that the dependence on $\tau,\lambda,d$ is quite reasonable in Proposition \ref{prop:knn}.  The argument requires the construction of a $\tau \lambda^2$-net on $M$ and it is not difficult to see that one needs a covering at scale proportional to $\tau \lambda$ in order to recover the intrinsic metric from discrete data points.  For example, consider Figure \ref{fig3}; were a curve to be added to connect the points at the bottleneck, this would drastically decrease the intrinsic distance between the bottleneck points.  In order to determine that the intrinsic distance between these points (without the connector) is actually quite large using the graph metric estimator, we need to set $\epsilon < \tau$, in which case these points are certainly only connected if there exists a point of distance less than $\tau$ to the bottleneck point, which can only occur with high probability if $n = \Omega\left(\tau^{-1}\right)$.  We can extend this example to arbitrary dimension $d$ by taking the product of the curve with $r S^{d-1}$ for $r = \Theta(\tau)$; in this case, a similar argument holds and we now need $\Omega\left(\tau^{-d}\right)$ points in order to guarantee with high probability that there exists a point of distance at most $\tau$ to one of the bottleneck points.  In this way, we see that the $\tau^{-d}$ scaling is unavoidable in general.  Note that the other estimators of intrinsic dimension mentioned in the introduction, in particular the MLE estimator of \cite{levina2004maximum}, implicitly require the accuracy of the $k$NN distance for their estimation to hold; thus these estimators also suffer from the $\tau^{-d}$ sample complexity.  Finally, we remark that \cite{kim2019minimax} presents a minimax lower bound for a related hypothesis testing problem and shows that minimax risk is bounded below by a local analogue of the reach raised to a power that depends linearly on the intrinsic dimension.

%% file: secLearning.tex
\section{Application of Techniques to GANs}\label{sec:gans}
In this section, we note that our techniques are not confined to the realm of dimension estimation and, in fact, readily apply to other problems.  As an example, consider the unsupervised learning problem of generative modeling, where we suppose that there are samples $X_1, \dots, X_n \sim \pp$ independent and we wish to produce a sample $\widehat{X} \sim \widehat{\pp}$ such that $\widehat{\pp}$ and $\pp$ are close.  Statistically, this problem can be expressed by fixing a class of distributions $\mathcal{P}$ and using the data to choose $\widehat{\mu} \in \mathcal{P}$ such that $\widehat{\mu}$ is in some sense close to $\pp$.  For computational reasons, one wishes $\mathcal{P}$ to contain distributions from which it is computationally efficient to sample; in practice, $\mathcal{P}$ is usually the class of pushforwards of a multi-variate Gaussian distribution by some deep neural network class $\mathcal{G}$.  While our statistical results include this setting, they are not restricted and apply for general classes of distributions $\mathcal{P}$.

In order to make the problem more precise, we require some notion of distance between distributions.  We use the notion of the Integral Probability Metric \citep{muller1997integral,sriperumbudur2012empirical} associated to a H{\"o}lder ball $C_B^\beta (\Omega)$, as defined above.  We suppose that $\supp \pp \subset \Omega$ and we abbreviate the corresponding IPM distance by $d_{\beta, B}$.  Given the empirical distribution $P_n$, the GAN that we study can be expressed as
\begin{equation*}
    \widehat{\mu} \in \argmin_{\mu \in \mathcal{P}} d_{\beta, B}(\mu, P_n) = \argmin_{\mu \in \mathcal{P}} \sup_{f \in C_B^\beta (\Omega)} \ee_\mu[f] - P_n f.
\end{equation*}

In this section, we generalize the results of \cite{schreuder2020statistical}.  In particular, we derive new estimation rates for a GAN using a H{\"o}lder ball as a discriminating class, assuming that the population distribution $\pp$ is low-dimensional; like \cite{schreuder2020statistical}, we consider the noised and potentially contaminated setting.  We have
\begin{theorem}\label{thm:gans}
  Suppose that $\pp$ is a probability measure on $\rr^D$ supported on a compact set $\mathcal{S}$ and suppose we have $n$ independent $X_i \sim \pp$ with empirical distribution $P_n$.  Let $\eta_i$ be independent, centred random variables on $\rr^D$ such that $\ee\left[\norm{\eta_i}^2\right] \leq \sigma^2$.  Suppose we observe $\widetilde{X}_i$ such that for at least $(1 - \epsilon) n$ of the $\widetilde{X}_i$, we have $\widetilde{X}_i = X_i + \eta_i$; let the empirical distribution of the $\widetilde{X}_i$ be $\widetilde{P}_n$.  Let $\mathcal{P}$ be a known set of distributions and define
  \begin{equation*}
      \widehat{\mu} \in \argmin_{ \mu \in \mathcal{P}} d_{\beta, B}(\mu, \widetilde{P}_n).
  \end{equation*}
  Then if there is some $C_1, d$ such that $N(\mathcal{S},\norm{\cdot}, \delta) \leq C_1 \epsilon^{-d}$, we have
  \begin{equation*}
      \ee\left[d_{\beta, B}(\widehat{\mu}, \pp)\right] \leq \inf_{\mu \in \mathcal{P}} d_{\beta, B}(\mu, \pp) + B(\sigma + 2 \epsilon) + C \beta B \sqrt{\log n} \left(n^{- \frac \beta d} \vee n^{- \frac 12}\right)
  \end{equation*}
  where $C$ is a constant depending linearly on $C_1$.
\end{theorem}
We note that the $\log n$ factor can be easily removed for all cases $\beta \neq \frac d2$ by paying slightly in order to increase the constants; for the sake of simplicity, we do not bother with this argument here.  The proof of Theorem \ref{thm:gans} is similar in spirit to that of \cite{schreuder2020statistical}, which in turn follows \cite{liang2018HowWell}, with details in Appendix \ref{app:miscellany}.  The key step is in applying the bounds in Lemma \ref{lemma:holderipm} to the arguments of \cite{liang2018HowWell}.

We compare our result to the corresponding theorem \cite[Theorem 2]{schreuder2020statistical}.  In that work, the authors considered a setting where there is a known intrinsic dimension $d$ and the population distribution $\pp = g_\# \mathcal{U}\left([0,1]^d\right)$, the push-forward by an $L$-Lipschitz function $g$ of the uniform distribution on a $d$-dimensional hypercube; in addition, they take $\mathcal{P}$ to be the set of push-forwards of $U\left([0,1]^d\right)$ by functions in some class $\F$, all of whose elements are $L$-Lipschitz.  Their result, \cite[Theorem 2]{schreuder2020statistical}, gives an upper bound of 
\begin{equation}\label{eq:schreuder}
    \ee\left[ d_{\beta, 1}(\widehat{\mu}, \pp)\right] \leq \inf_{ \mu \in \mathcal{P}} d_{\beta, 1}(\mu, \pp) +L(\sigma + 2 \epsilon)  + c L \sqrt{d} \left(n^{- \frac \beta d} \vee n^{-\frac 12}\right).
\end{equation}
Note that our result is an improvement in two key respects.  First, we do not treat the intrinsic dimension $d$ as known, nor do we force the dimension of the feature space to be the same as the intrinsic dimension.  Many of the state-of-the-art GAN architectures on datasets such as ImageNet use a feature space of dimension 128 or 256 \citep{wu2019logan}; the best rate that the work of \cite{schreuder2020statistical} can give, then would be $n^{- \frac 1{128}}$.  In our setting, even if the feature space is complex, if the true distribution lies on a much lower dimensional subspace, then it is the true, intrinsic dimension, that determines the rate of estimation.  Secondly, note that the upper bound in \eqref{eq:schreuder} depends on the Lipschitz constant $L$; as the function classes used to determine the push-forwards are essentially all deep neural networks in practice, and the Lipschitz constants of such functions are exponential in depth, this can be a very pessimistic upper bound; our result, however, does not depend on this Lipschitz constant, but rather on properties intrinsic to the probability distribution $\pp$.  This dependence is particularly notable in the noisy regime, where $\sigma, \epsilon$ do not vanish; the large multiplicative factor of $L$ in this case would then make the bound useless. 

We conclude this section by considering the case most often used in practice: the Wasserstein GAN.
\begin{corollary}\label{cor:gans}
    Suppose we are in the setting of Theorem \ref{thm:gans} and $\mathcal{S}$ is contained in a ball of radius $R$ for $R \geq \frac 12$.  Then,
    \begin{align*}
        \ee\left[W_1(\widehat{\mu}, \pp)\right] \leq \inf_{\mu \in \mathcal{P}} W_1(\mu, \pp) + \sigma + 2 R \epsilon +  C R \sqrt{\log n} n^{- \frac 1d}.
    \end{align*}
\end{corollary}
The proof of the corollary is almost immediate from Theorem \ref{thm:gans}.  With additional assumptions on the tails of the $\eta_i$, we can turn our expectation into a high probability statement.  In the special case with neither noise nor contamination, i.e. $\sigma = \epsilon = 0$, we get that the Wasserstein GAN converges in Wasserstein distance at a rate of $n^{- \frac 1d}$, which we believe explains in large part the recent empirical success in modern Wasserstein-GANs.

%% file: proof-coverings.tex
\section{Proofs from Section \ref{sec:covering}}\label{app:coverings}
\begin{proof}[Proof of Proposition \ref{prop:keyHolderprop}]
  	We apply the method from the classic paper \citep{Tikhomirov1993}, following notation introduced there as applicable.  For the sake of simplicity, we assume that $\beta$ is an integer; the generalization to $\beta \not\in \mathbb{N}$ is analogous to that in \cite{Tikhomirov1993}.  Let $\Delta^\beta = \frac{\epsilon}{2 B}$ and let $x_1, \dots, x_s$ be a $\Delta$-connected $\Delta$ net on $\cs$.  For $0 \leq k \leq \beta$ and $1 \leq i \leq s$, define
		\begin{align*}
			\gamma^k(f) = \left\lfloor\frac{\norm{D^k f(x_i)}}{\epsilon_k} \right\rfloor && \epsilon_k = \frac{\epsilon}{\Delta^k}
		\end{align*}
		where $\norm{\cdot}$ is the norm on tensors induced by the ambient (Euclidean) metric and $D^k$ is the $k^{th}$ application of the covariant derivative.  Let $\gamma(f)= \left(\gamma_i^k(f)\right)_{i,k}$ be the matrix of all $\gamma_i^k(f)$ and let $U_\gamma$ be the set of all $f$ such that $\gamma(f) = \gamma$.  Then the argument in the proof of \cite[Theorem XIV]{Tikhomirov1993} applies \emph{mutatis mutandis} and we note that $U_\gamma$ are $2 \epsilon$ neighborhoods in the H{\"o}lder norm.  Thus it suffices to bound the number of possible $\gamma$.  As in \cite{Tikhomirov1993}, we note that the number of possible values for $\gamma_1^k$ is at most $\frac{2 B}{\epsilon_k}$.  Given the row $\left(\gamma_i^k\right)_{0 \leq k \leq \beta}$, there are at most $(4 e + 2)^{\beta + 1}$ values for the next row.  Thus the total number of possible $\gamma$ is bounded by
		\begin{align*}
			\left((4 e + 2)^{\beta+1}\right)^s\prod_{k = 1}^\beta \frac{2 B}{\epsilon_k} &= (4 e + 2)^{(\beta + 1)s} \prod_{k = 1}^\beta \frac{2 B}{\epsilon} \left(\frac \epsilon{2B}\right)^{\frac k\beta} = (4e + 2)^{(\beta + 1)s} \left(\frac {2B}\epsilon\right)^{\frac{\beta}{2}}.
		\end{align*}
		By definition of the covering number and the fact that $\cs$ is path-connected, we may take
		\begin{equation*}
		    s = N(\cs, \Delta) = N\left(\cs, \left(\frac{\epsilon}{2 B}\right)^{\frac 1\beta}\right).
		\end{equation*}
		Taking logarithms and noting that $\log(4e + 2) \leq 3$ concludes the proof of the upper bound.
		
		The middle inequality is \Cref{lem:duality}.  For the lower bound, we again follow \cite{Tikhomirov1993}.  Define
		\begin{equation*}
		    \phi(x) = \begin{cases}
		    
		      a \prod_{i = 1}^D \left(1 - x_i^2\right)^{\frac \beta 2} & \norm{x}_\infty \leq 1 \\
		      0 & \text{otherwise}
		    \end{cases}
		\end{equation*}
		with $a$ a constant to be set.  Choose a $2 \Delta$-separated set $x^1, \dots, x^ss$ with $\Delta = \left(\frac \epsilon {2B}\right)^{\frac 1\beta}$ and consider the set of functions
		\begin{equation*}
		    g_{\mathbf{\sigma}} = \sum_{i = 1}^s \sigma_i \Delta^\beta \phi\left(\frac{x - x^i}{\Delta}\right)
		\end{equation*}
		where $\sigma_i \in \{\pm 1\}$ and $\mathbf{\sigma}$ varies over all possible sets of signs.  The results of \cite{Tikhomirov1993} guarantee that the $g_{\mathbf{\sigma}}$ form a $2 \epsilon$-separated set in $\F$ if $a$ is chosen such that $g_{\mathbf{\sigma}} \in \F$ and there are $2^s$ such combinations.  By definition of packing numbers, we may choose
		\begin{equation*}
		    s = D\left(\F, \left(\frac{\epsilon}{B}\right)^{\frac 1\beta}\right).
		\end{equation*}
		This concludes the proof of the lower bound.
\end{proof}

\begin{proof}[Proof of Proposition \ref{prop:mfldcovering}]
        We note first that the second statement follows from the first by applying (b) and (c) to Proposition \ref{prop:mfldgeometry} to control the curvature and injectivity radius in terms of the reach.  Furthermore, the middle inequality in the last statement follows from \Cref{lem:duality}.  Thus we prove the first two statements.
        
		A volume argument yields the following control:
		\begin{equation*}
			N\left(M, \norm{\cdot}_g, r\right) \leq \frac{\vol M}{\inf_{p \in M} \vol B_{\frac \epsilon2}(p)}
		\end{equation*}
		where $B_{\frac \epsilon2}(p)$ is the ball around $p$ of radius $\frac \epsilon2$ with respect to the metric $g$.  Thus it suffices to lower bound the volume of such a ball.  Because $\epsilon <  \iota$, we may apply the Bishop-Gromov comparison theorem \cite[Theorem 3.17]{Gray2004} to get that
		\begin{equation*}
			\vol B_\epsilon (p) \geq \frac{2 \pi^{\frac d2}}{\Gamma\left(\frac d2\right)} \int_0^\epsilon \left(\frac{\sin\left(t \sqrt{\kappa_1}\right)}{\sqrt{\kappa_1}}\right)^{d-1} d t = \omega_d \int_0^\epsilon  \left(\kappa_1^{- \frac 12} \sin\left(t \sqrt{\kappa_1}\right)\right)^{d-1} d t
		\end{equation*}
		where $\kappa_1$ is an upper bound on the sectional curvature.  We note that for $t \leq \frac \pi {2\sqrt{\kappa_1}}$, we have $\sin\left(t \sqrt{\kappa_1}\right) \geq \frac{2}{\pi} t\sqrt{\kappa_1} $ and thus
		\begin{equation*}
			\vol B_\epsilon (p) \geq \omega_d \int_0^\epsilon \left(\frac 2 \pi t\right)^{d-1} d t = \frac{\omega_d}{d} \left(\frac 2\pi\right)^{d-1} \epsilon^d.
		\end{equation*}
		The upper bound follows from control on the sectional curvature by $\tau$, appearing in \cite[Proposition A.1]{aamari2019estimating}, which, in turn, is an easy consequence of applying the Gauss formula to (a) of Proposition \ref{prop:mfldgeometry}.
		
		We lower bound the packing number through an analogous argument as the upper bound for the covering number, this time with an upper bound on the volume of a ball of radius $\epsilon$, again from \cite[Theorem 3.17]{Gray2004}, but this time using a lower bound on the sectional curvature.  In particular, we have for $\epsilon < \iota$,
		\begin{equation*}
		    \vol B_\epsilon(p) \leq \omega_d \int_0^\epsilon \left(\frac{\sin\left(t \sqrt{\kappa_2}\right)}{\sqrt{\kappa_2}}\right)^{d-1} d t = \omega_d \int_0^\epsilon \left( \frac{\sinh\left(t \sqrt{- \kappa_2}\right)}{\sqrt{- \kappa_2}} \right)^{d-1} d t
		\end{equation*}
		where $\kappa_2$ is a lower bound on the sectional curvature.  Note that for $t \leq \frac{1}{\sqrt{-\kappa_2}}$, we have
		\begin{equation*}
		    \frac{\sinh\left(t \sqrt{- \kappa_2}\right)}{\sqrt{- \kappa_2}}  \leq \cosh(2) t \leq 4 t.
		\end{equation*}
		Thus,
		\begin{equation*}
		    \vol B_\epsilon(p) \leq \omega_d \int_0^\epsilon (4 t)^{d-1} d t = \frac{\omega_d}{d} 4^d \epsilon^d.
		\end{equation*}
		The volume argument tells us that
		\begin{equation*}
			N\left(M, \norm{\cdot}_g, r\right) \geq \frac{\vol M}{\sup_{p \in M} \vol B_r(p)}
		\end{equation*}
		and the result follows.
		
		If we wish to extend the range of $\epsilon$, we pay with a constant exponential exponential in $d$, reflecting the growth in volume of balls in negatively curved spaces.  In particular, we can apply the same argument and note that as $\frac{\sinh(x)}{x}$ is increasing, we have
		\begin{equation*}
		    \frac{\sinh\left(t \sqrt{- \kappa_2}\right)}{\sqrt{- \kappa_2}} \leq \frac{\sinh(\iota \sqrt{- \kappa_2})}{\iota \sqrt{-\kappa_2}} t \leq \frac{e^{\iota \sqrt{- \kappa_2}}}{\iota \sqrt{-\kappa_2}} t
		\end{equation*}
		for all $t < \iota$.  Thus for all $\epsilon < \iota$.  We have:
		\begin{equation*}
		    N(M, \norm{\cdot}_g, \epsilon) \geq \frac{\vol M}{\omega_d} d \iota^d (- \kappa_2)^{\frac d2} e^{- d \iota \sqrt{- \kappa_2}} \epsilon^{-d}
		\end{equation*}
		as desired.
\end{proof}
\begin{proof}[Proof of \Cref{cor:ambientcovering}]
  Let $B^{\rr^D}_\epsilon(p)$ be the set of points in $\rr^D$ with Euclidean distance to $p$ less than $\epsilon$ and let $B^M_\epsilon(p)$ be the set of points in $M$ with intrinsic (geodesic) distance to $p$ less than $\epsilon$.  Then, if $\epsilon \leq 2 \tau$, combining the fact that straight lines are geodesics in $\rr^D$ and (d) from Proposition \ref{prop:mfldgeometry} gives
  \begin{equation*}
      B_\epsilon^M(p) \subset B_\epsilon^{\rr^D}(p) \cap M \subset B_{2 \tau \arcsin\left(\frac \epsilon{2 \tau}\right)}^M(p)
  \end{equation*}
  In particular, this implies
  \begin{align*}
      N\left( M, d_M, 2 \tau \arcsin\left(\frac \epsilon{2\tau}\right)\right) &\leq N(M, \norm{\cdot}, \epsilon) \leq N(M, d_M, \epsilon) \\
      D\left( M, d_M, 2 \tau \arcsin\left(\frac \epsilon{2\tau}\right)\right) &\leq D(M, \norm{\cdot}, \epsilon) \leq D(M, d_M, \epsilon) 
  \end{align*}
  whenever $\epsilon \leq 2\tau$.  Thus, applying Proposition \ref{prop:mfldcovering}, we have
  \begin{align*}
      N(M, \norm{\cdot}, \epsilon) \leq N(M, d_M, \epsilon) \leq \frac{\vol M}{\omega_d} d \left(\frac \pi 2\right)^d \epsilon^{-d}
  \end{align*}
  and similarly,
  \begin{align*}
      D(M, \norm{\cdot}, 2\epsilon) \geq D\left( M, d_M, 2 \tau \arcsin\left(\frac \epsilon{\tau}\right)\right) \geq \frac{\vol M}{\omega_d} d 16^{-d}  \epsilon^{-d}
  \end{align*}
  using the fact that $\arcsin(x) \leq 2x$ for $x \geq 0$.  The result follows.
\end{proof}

%% file: proof-dimestim.tex
\section{Proof of Proposition \ref{prop:t2control}}\label{app:t2control}
As stated in the body, we bound the $T_2$ constant $c_2$ by the log-Sobolev constant of the same measure.  We thus first define a log-Sobolev inequality:
\begin{definition}
    Let $\mu$ be a measure on $M$.  We say that $\mu$ satisfies a log-Sobolev inequality with constant $c_{LS}$ if for all real valued, differentiable functions with mean 0 $f: M \to \rr$, we have:
    \begin{equation*}
        \int_M f^2 \log(f^2) d \mu \leq c_{LS} \int_M \norm{\nabla f}^2 d \mu
    \end{equation*}
    where $\nabla$ is the Levi-Civita connection and $\norm{\cdot}$ is the norm with respect to the Riemannian metric.
\end{definition}
While in the main body we cited \cite{otto2000generalization} for the Otto-Villani theorem, we actually need a slight strengthening of this result.  For technical reasons, \cite{otto2000generalization} required the density of $\mu$ to have two derivatives; more recent works have eliminated that assumption.  We have:
\begin{theorem}[Theorem 5.2 from \cite{gigli2013log}]\label{thm:ottovillani}
    Suppose that $\mu$ satisfies the log-Sobolev inequality with constant $c_{LS}$.  Then $\mu$ satisfies the $T_2$ inequality with constant $c_2 \leq 2 c_{LS}$.
\end{theorem}
We now recall the key estimate from \cite{wang1997estimation} that controls the log-Sobolev constant for the uniform measure on a compact manifold $M$\footnote{We remark that some works, including \cite{wang1997estimation}, define the log-Sobolev constant to be the inverse of our $c_{LS}$.  We translate their theorem into our terms by taking the recipricol.}:
\begin{theorem}[Theorem 3.3 from \cite{wang1997estimation}] \label{thm:wang}
    Let $M$ be a compact, $d$-dimensional manifold with diameter $\Delta$.  Suppose that $\ric_M \succeq - K$ for some $K \in \rr$.  Let $\mu$ be the uniform measure on $M$ (i.e., the volume measure normalized so that $\mu(M) = 1$).  Then $\mu$ satisfies a log-Sobolev inequality with
    \begin{equation*}
        c_{LS} \leq \left(\frac{d + 2}{d}\right)^d \frac{e^{2 K (d+1)\Delta^2} - 1}{K} e^{1 + d \Delta^2 K_+}.
    \end{equation*}
\end{theorem}
We are now ready to complete the proof.
\begin{proof}[Proof of Proposition \ref{prop:t2control}]
    By the Holly-Stroock perturbation theorem \citep{holley1986logarithmic}, we know that if $\mu$ is the uniform measure on $M$ normalized such that $\mu(M) = 1$, and $\mu$ satisfies a log-Sobolev inequality with constant $c_{LS}'$ then $\pp$ satisfies a log-Sobolev constant with $c_{LS} \leq \frac{W}{w} c_{LS}'$.  By (a) from Proposition \ref{prop:mfldgeometry}, we have that the sectional curvatures of $M$ are all bounded below by $- \frac 2{\tau^2}$ and thus $\ric_M \succeq - (d - 1) \frac{2}{\tau^2}$ (for the relationship between the Ricci tensor and the sectional curvatures, see \cite{lee2013smooth}).  Noting that $\frac{d + 2}{d} \leq 3$ and plugging into the results of Theorem \ref{thm:wang}, we get that
    \begin{equation*}
        c_{LS}' \leq \frac{2\tau^2}{d-1} \exp\left(d \log 3 + \frac{3 \Delta^2 d^2}{\tau^2}\right).
    \end{equation*}
    Combining this with the Holly-Stroock result and Theorem \ref{thm:ottovillani} concludes the proof.
\end{proof}
\section{Proof of Theorem \ref{thm:estimator}}\label{app:dimestim}

We first prove the following lemma on the concentration of $W_1(P_n, P_n')$.
\begin{lemma}\label{lem:w1concentration}
    Suppose that $\pp$ is a probability measure on $(T, d)$ and that it satisfies a $T_2(c_2)$-inequality.  Let $X_1, \dots, X_n, X_1', \dots, X_n'$ denote independent samples with corresponding empirical distributions $P_n, P_n'$.  Then the following inequalities hold:
    \begin{align*}
        \pp\left(\abs{W_1(P_n, P_n') - \ee \left[W_1(P_n, P_n')\right]} \geq t\right) &\leq 2e^{- \frac{n t^2}{8 c_2}} \\
        \pp\left(\abs{W_1(P_n, P_n') - \ee \left[W_1(P_n, P_n')\right]} \leq t\right) &\leq 2e^{- \frac{n t^2}{8 c_2}}.
    \end{align*}
\end{lemma}
\begin{proof}
    We note that by \cite{gozlan2009characterization}, in particular the form of the main theorem stated in \cite[Theorem 4.31]{vanHandel2014}, it suffices to show that, as a function of the data, $W_1(P_n, P_n')$ is $\frac 2{\sqrt n}$-Lipschitz.  Note that by symmetry, it suffices to show a one-sided inequality.  By the triangle inequality,
    \begin{equation*}
        W_1(P_n, P_n') \leq W_1(P_n, \mu) + W_1(P_n', \mu)
    \end{equation*}
    for any measure $\mu$ and thus it suffices to show that $W_1(P_n, \mu)$ is $\frac 1{\sqrt n}$-Lipschitz in the $X_i$.  By \cite[Lemma 4.34]{vanHandel2014}, there exists a bijection between the set of couplings between $P_n$ and $\mu$ and the set of ordered $n$-tuples of measures $\mu_1, \dots, \mu_n$ such that $\mu = \frac 1n \sum_i \mu_i$.  Thus we see that if $X, \widetilde{X}$ are two data sets, then
    \begin{align*}
        W_1(P_n, \mu) - W_1(\widetilde{P}_n, \mu) &\leq \sup_{\frac 1n \sum_{i = 1}^n \mu_i = \mu} \left[\frac 1n \sum_{i = 1}^n \int \left(d(X_i, y) - d(\widetilde{X}_i, y) \right) d \mu_i(y) \right] \\
        &\leq \sup_{\frac 1n \sum_{i = 1}^n \mu_i = \mu} \left[ \frac 1n \sum_{i = 1}^n \int d(X_i, \widetilde{X}_i) d \mu_i(y)      \right] \\
        &= \frac 1n \sum d(X_i, \widetilde{X}_i) \\
        &\leq \frac 1n \sqrt{n \sum_{i = 1}^n d(X_i, \widetilde{X}_i)^2} \leq \frac 1{\sqrt n} d^{\otimes n}(X, \widetilde{X}).
    \end{align*}
    The identical argument applies to $W_1^M$.
\end{proof}
We are now ready to show that $\widehat{d}_n$ is a good estimator of $d$.
\begin{proposition}\label{prop:dhat}
    Suppose we are in the situation of Theorem \ref{thm:estimator} and we have 
    \begin{align*}
        n &\geq \max \left( \frac{ d \vol M}{4 w\omega_d} \left(\frac \iota 8\right)^{-d}, \left(\frac{8 c_2}{\Delta^2} \log \frac 1\rho\right)^{\frac{d}{2d - 5}}\right) \\
        \alpha &\geq \max\left(\log^{\frac d{2\gamma}}\left(\frac{n \omega_d \Delta^d}{d \vol M}\right), (Cw)^{\frac 1\gamma}\right)
    \end{align*}
    Then with probability at least $1 - 4\rho$, we have
    \begin{equation*}
        \frac{d}{1 + 3 \gamma} \leq \widehat{d}_n \leq (1 + 3 \gamma)d.
    \end{equation*}
\end{proposition}
\begin{proof}
  By Proposition \ref{prop:w1} and Lemma \ref{lem:w1concentration}, we have that with probability at least $1 - e^{- \frac{n t^2}{8 c_2}}$, we have
  \begin{align*}
      W_1^M(P_n, P_n') \leq C \left(\frac{ \vol M}{n \omega_d}\right)^{\frac 1d} \sqrt[]{\log\left(\frac{n \omega_d}{d \vol_M}\right)} + t.
  \end{align*}
  By Proposition \ref{prop:w1lower} and Lemma \ref{lem:w1concentration} and the left hand side of Proposition \ref{prop:w1}, we have that with probability at least $1 - e^{- \frac{\alpha n t^2}{8 c_2}}$,
  \begin{equation*}
    W_1^M(P_{\alpha n}, P_{\alpha n}') \geq \frac 1{32}\left(\frac{ d \vol M}{4 w \omega_d}\right)^{\frac 1d}  (\alpha n)^{- \frac 1d} - t
  \end{equation*}
  all under the assumption that
  \begin{equation*}
    n > \frac{ d \vol M}{4 w\omega_d} \left(\frac \iota 8\right)^{-d}.
  \end{equation*}
  Setting $t = \Delta (\alpha n)^{- \frac {5}{4d}}$, we see that, as $\alpha > 1$, with probability at least $1 - 2 e^{- \frac{n t^2}{8 c_2}}$, we simultaneously have
  \begin{align*}
      W_1^M(P_n, P_n') &\leq C  \left(\frac{ \vol M}{n \omega_d}\right)^{\frac 1d} \sqrt[]{\log\left(\frac{n \omega_d \Delta^d}{d \vol_M}\right)} \\
      W_1^M(P_{\alpha n}, P_{\alpha n}') &\geq \frac 1{64}\left(\frac{ d \vol M}{4 w \omega_d}\right)^{\frac 1d}  (\alpha n)^{- \frac 1d}.
  \end{align*}
  Thus, in particular,
  \begin{align*}
      \frac{W_1^M(P_n, P_n')}{W_1^M(P_{\alpha n}, P_{\alpha n}')} \leq \frac{C \ \left(\frac{ \vol M}{n \omega_d}\right)^{\frac 1d} \sqrt[]{\log\left(\frac{n \omega_d}{d \vol M}\right)}}{\frac 1{64}\left(\frac{ d \vol M}{4 w \omega_d}\right)^{\frac 1d}  (\alpha n)^{- \frac 1d}} \leq C w^{\frac 1d} \alpha^{\frac 1d} \sqrt[]{\log\left(\frac{n \omega_d \Delta^d}{d \vol M}\right)}  
  \end{align*}
  Thus we see that
  \begin{align*}
      \widehat{d}_n &= \frac{\log \alpha}{\log \frac{W_1(P_n, P_n')}{W_1(P_{\alpha n}, P_{\alpha n}')}} \\
      &\geq \frac{\log \alpha}{\frac 1d \log \alpha + + \frac 1d \log w + \frac 12 \log\log\left(\frac{n \omega_d \Delta^d}{d \vol M}\right)} \\
      &= \frac{d}{1 + \frac{\log(C w) + \frac d2 \log\log\left(\frac{n \omega_d \Delta^d}{d \vol M}\right)}{\log \alpha}}
  \end{align*}
  Now, if
  \begin{align*}
      n &\geq \max \left( \frac{ d \vol M}{4 w\omega_d} \left(\frac \tau 8\right)^{-d}, \left(\frac{8 c_2^2}{\Delta^2} \log \frac 1\rho\right)^{\frac{d}{2d - 5}}\right) \\
      \alpha &\geq \max\left(\log^{\frac d{2\gamma}}\left(\frac{n \omega_d \Delta^d}{d \vol M}\right), (Cw)^{\frac 1\gamma}\right)
  \end{align*}
  Then with probability at least $1 - 2\rho$,
  \begin{equation*}
      \widehat{d}_n \geq \frac{d}{1 + 2 \gamma}.
  \end{equation*}
  An identical proof holds for the other side of the bound and thus the result holds.
\end{proof}
We are now ready to prove the main theorem using Proposition \ref{prop:dhat} and Proposition \ref{prop:knn}.
\begin{proof}[Proof of Theorem \ref{thm:estimator}]
  Note first that
  \begin{align}
      w\left(\frac{\iota \lambda^2}{8}\right) &\geq  \frac{w \omega_d}{d} \left(\frac \pi 2\right)^{-d} \left(\frac{\iota \lambda^2}{8}\right)^{d} \label{eq:mainthmproof1} \\ 
      N\left(M, d_M, \frac{\iota \lambda^2}{8} \right) &\leq \frac{\vol M}{\omega_d} d \left(\frac \pi 2\right)^d \left( \frac{\iota \lambda^2}{8} \right)^{-d} \label{eq:mainthmproof2}
  \end{align}
  by Proposition \ref{prop:mfldcovering}.  Setting $\lambda = \frac 12$, we note that by Proposition \ref{prop:knn}, if the total number of samples
  \begin{align*}
      2(\alpha + 1) n \geq \left(\frac{w \omega_d}{d} \left(\frac \pi 2\right)^{-d} \left(\frac{\iota \lambda^2}{8}\right)^{d}\right)^{-1} \log \left(\frac{\vol M}{\rho \omega_d} d \left(\frac{\tau}{16 \pi} \right)^{-d}\right)
  \end{align*}
  then with probability at least $1 - \rho$, we have
  \begin{align*}
      \frac 12 d_M(p, q) \leq d_G(p, q) \leq \frac 32 d_M(p, q)
  \end{align*}
  for all $p, q \in M$.  Thus by the proof of Proposition \ref{prop:dhat} above,
  \begin{equation*}
      \frac{W_1^M(P_n, P_n')}{W_1^M(P_{\alpha n}, P_{\alpha n}')} \leq \frac{1 + \lambda}{1 - \lambda} C w^{\frac 1d} \alpha^{\frac 1d} \sqrt[]{\log\left(\frac{n \omega_d \Delta^d}{d \vol M}\right)} .
  \end{equation*}
  Thus as long as $\alpha \geq \left(\frac{1 + \lambda}{1 - \lambda}\right)^{\frac d\gamma} = 3^{\frac d\gamma}$, then we have with probability at least $ 1 - 3\rho$,
  \begin{equation*}
      \widetilde{d}_n \geq \frac{d}{1 + 3 \gamma}.
  \end{equation*}
  A similar computation holds for the other bound.
  
  To prove the result for $d_n$, note that if we replace the $\iota$s by $\tau$ in \eqref{eq:mainthmproof1} and \eqref{eq:mainthmproof2}, then the result still holds by the second part of Proposition \ref{prop:mfldcovering}.  Then the identical arguments apply, \emph{mutatis mutandis}, after skipping the step of approximating $d_M$ by $d_G$.
\end{proof}

%% file: proof-metric.tex
\section{Metric Estimation Proofs}\label{app:metricestim}

In order to state our result, we need to consider the minimal amount of probability mass that $\pp$ puts on any intrinsic ball of a certain radius in $M$.  To formalize this notion, we define, for $\delta > 0$, 
\begin{equation*}
	w_B(\delta) = \inf_{p \in M} \pp\left(B_\delta^M(p)\right).
\end{equation*}

We need a few lemmata:
	\begin{lemma}\label{lem:knn1}
		Fix $\epsilon > 0$ and a set of $x_i \in M$ and form $G(x, \epsilon)$.  If the set of $x_i$ form a $\delta$-net for $M$ such that $\delta \leq \frac \epsilon 4$, then for all $x,y \in M$,
		\begin{equation*}
			d_G(x,y) \leq \left(1 + \frac {4\delta}\epsilon\right) d_M(x,y).
		\end{equation*}
	\end{lemma}
	\begin{proof}
		This is a combination of \cite[Proposition 1]{Bernstein00} and \cite[Theorem 2]{Bernstein00}.
	\end{proof}
	\begin{lemma}\label{lem:knn2}
		Let $0 < \lambda < 1$ and let $x,y \in M$ such that $\norm{x - y} \leq 2 \tau \lambda(1-\lambda)$.  Then
		\begin{equation*}
			(1 - \lambda) d_M(x,y) \leq \norm{x - y} \leq d_M(x,y).
		\end{equation*}
	\end{lemma}
	\begin{proof}
		Note that $2 \tau \lambda ( 1 - \lambda) \leq \frac \tau2$ so we are in the situation of Proposition \ref{prop:mfldgeometry} (e).  Let $\ell = d_M(x,y)$.  Rearranging the bound in Proposition \ref{prop:mfldgeometry} (e) yields
		\begin{equation*}
			\ell\left(1 - \frac \ell{2 \tau}\right) \leq \norm{x -y} \leq \ell.
		\end{equation*}
		Thus it suffices to show that
		\begin{equation*}
			\frac{\ell}{2 \tau} \leq \lambda.
		\end{equation*}
		Again applying Proposition \ref{prop:mfldgeometry}, we see that
		\begin{equation*}
			\ell \leq \tau\left(1 - \sqrt{1 - \frac{2 \norm{x-y}}{\tau}}\right).
		\end{equation*}
		Rearranging and plugging in $\norm{x - y} \leq 2 \tau \lambda (1 - \lambda)$ concludes the proof.
	\end{proof}
	The next lemma is a variant of \cite[Lemma 5.1]{niyogi2008finding}.
	\begin{lemma}\label{lem:knn3}
		Let $w_B(\delta)$ be as in Proposition \ref{prop:knn} and let $N(M, \delta)$ be the covering number of $M$ at scale $\delta$.  If we sample $n \geq w\left(\frac \delta 2\right)^{-1} \log \frac{N\left(M,\frac \delta2\right)}{\rho}$ points independently from $\pp$, then with probability at least $1 - \rho$, the points form a $\delta$-net of $M$.
	\end{lemma}
	\begin{proof}
		Let $y_1, \dots, y_N$ be a minimal $\frac \delta 2$-net of $M$.  For each $y_i$ the probability that $x_i$ is not in $B_{\frac \delta 2}(y_i)$ is bounded by $1 - w_B\left(\frac \delta 2\right)$ by definition.  By independence, we have
		\begin{equation*}
			\pp\left(\forall i \,\, x_j \not\in B_{\frac \delta 2}(y_i)\right) \leq \left(1 - )Bw\left(\frac \delta 2\right)\right)^n \leq e^{- n w_B\left(\frac \delta 2\right)}.
		\end{equation*}
		By a union bound, we have
		\begin{equation}\label{eq:knn3}
			\pp\left(\exists i \text{ such that } \forall j \,\, x_j \not\in B_{\frac \delta 2}(y_i)\right) \leq N\left(M,\frac \delta2\right) e^{ -  n w_B\left(\frac \delta 2\right)}.
		\end{equation}
		If $n$ satisfies the bound in the statement then the right hand side \eqref{eq:knn3} is controlled by $\rho$.
	\end{proof}
	Note that for any measure $\pp$, a simple union bound tells us that $w_B(\delta) \leq N\left(M, \delta\right)^{-1}$ and that equality, up to a constant, is achieved for the uniform measure.  This is within a log factor of the obvious lower bound given by the covering number on the number of points required to have a $\delta$-net on $M$.
	
	With these lemmata, we are ready to conclude the proof:
	\begin{proof}[Proof of Proposition \ref{prop:knn}]
		Let $\epsilon = \tau \lambda \leq 2\tau \lambda(1 - \lambda)$ by $\lambda \leq \frac 12$.  Let $\delta = \frac{\lambda \epsilon}{4} = \frac{\tau \lambda^2}{4}$.  By Lemma \ref{lem:knn3}, with high probability, the $x_i$ form a $\delta$-net on $M$; thus for the rest of the proof, we fix a set of $x_i$ such that this condition holds.  Now we may apply Lemma \ref{lem:knn1} to yield the upper bound $d_G(x,y) \leq (1 + \lambda) d_M(x,y)$.
		
		For the lower bound, for any points $p,q \in M$ there are points $x_{j_0}, x_{j_m}$ such that $d_M(p, x_{j_0}) \leq \delta$ and $d_M(q, x_{j_m}) \leq \delta$ by the fact that the $x_i$ form a $\delta$-net.  Let $x_{j_1}, \dots, x_{j_{m-1}}$ be a geodesic in $G$ between $x_{j_0}$ and $x_{j_m}$.  By Lemma \ref{lem:knn2} and the fact that edges only exist for small weights, we have
		\begin{align*}
			d_M(p,q) &\leq d_M\left(p, x_{j_0}\right) + d_M\left(x_{j_m}, q\right) + \sum_{i = 1}^m d_M\left(x_{j_{i-1}}, x_{j_i}\right) \\
			&\leq (1 - \lambda)^{-1} \left(\norm{p - x_{j_0}} + \norm{x_{j_m} - q} + \sum_{i = 1}^m \norm{x_{j_{i-1}}- x_{j_i}}\right)\\
			&= (1 - \lambda)^{-1} d_G(p,q).
		\end{align*}
		Rearranging concludes the proof.
\end{proof}

%% file: proof-miscellany.tex
\section{Miscellany}\label{app:miscellany}

\begin{proof}[Proof of Lemma \ref{lemma:holderipm}]
  By symmetrization and chaining, we have
  \begin{align*}
      \ee\left[ \sup_{f \in \F} \frac 1n \sum_{i = 1}^n f(X_i) - f(X_i')\right]  &\leq  2 \ee\left[ \sup_{f \in \F} \frac 1n \sum_{i = 1}^n \epsilon_i f(X_i)\right] \leq 2 \inf_{\delta > 0} \left[8 \delta + \frac{8 \sqrt 2}{\sqrt n} \int_\delta^{B} \sqrt{\log N(\F, \norm{\cdot}_\infty, \epsilon)} d \epsilon     \right] \\
      &\leq 2 B \inf_{\delta > 0}\left[ 8 \delta + \frac{8 \sqrt 2}{\sqrt n} \int_{\delta}^1 \sqrt{\log N\left(\F, \norm{\cdot}_\infty, \frac \epsilon{2R}\right)} d \epsilon \right] \\
      &\leq 2 B \inf_{\delta > 0}\left[ 8 \delta + \frac{8 \sqrt 2}{\sqrt n} \int_{\delta}^1 \sqrt{3\beta^2 \log \frac 1\epsilon N(\mathcal{S}, \norm{\cdot}, \epsilon)} d \epsilon \right]
  \end{align*}
  where the last step follows from Proposition \ref{prop:keyHolderprop}.  The first statement follows from noting that $\sqrt{\log \frac 1\epsilon}$ is decreasing in $\epsilon$, and thus allowing it to be pulled from the integral.  If $\beta > \frac d2$, the second statement follows from plugging in $\delta = 0$ and recovering a rate of $n^{- \frac 12}$.  If $\beta < \frac d2$, then the second statement follows from plugging in $\delta = n^{- \frac \beta d}$.
\end{proof}

\begin{proof}[Proof of Proposition \ref{prop:w1lower}]
  We follow the proof of \cite[Proposition 6]{weed2019sharp} and use their notation.  In particular, let
  \begin{equation*}
      N_\epsilon\left(\pp, \frac 12\right) = \inf\left\{N(S,d_M, 
    \epsilon)| S \subset M \text{ and } \pp(S) \geq \frac 12\right\}.
  \end{equation*}
  Applying a volume argument in the identical fashion to Proposition \ref{prop:mfldcovering}, but lower bounding the probability of a ball of radius $\epsilon$ by $w$ multiplied by the volume of said small ball, we get that
  \begin{equation*}
      N_\epsilon\left(\pp, \frac 12\right) \geq \frac{\vol M}{2w \omega_d} d 8^{- d} \epsilon^{-d}
  \end{equation*}
  if $\epsilon \leq \tau$.  Let
  \begin{equation*}
      \epsilon = \left( \frac{\vol M}{4 w \omega_d} d 8^{-d}\right)^{\frac 1d}  n^{- \frac 1d}
  \end{equation*}
  and assume that
  \begin{equation*}
      n >  \frac{\vol M}{4 w\omega_d} d 8^{-d}\left(\tau\right)^{- d}
  \end{equation*}
  Let
  \begin{equation*}
      S = \bigcup_{1 \leq i \leq n} B_{\frac \epsilon2}^M(X_i).
  \end{equation*}
  Then because
  \begin{equation*}
      N_\epsilon \left( \pp, \frac 12 \right) > n
  \end{equation*}
  by our choice of $\epsilon$, we have that $\pp(S) < \frac 12$.  Thus if $X \sim \pp$ then we have with probability at least $\frac 12$, $d_M(X, \{X_1, \dots, X_n\}) \geq \frac \epsilon 2$.  Thus the Wasserstein distance between $\pp$ and $P_n$ is at least $\frac \epsilon 4$.  The first result follows.  We may apply the identical argument, instead using intrinsic covering numbers and the bound in Proposition \ref{prop:mfldcovering} to recover the second statement. 
\end{proof}

\begin{proof}[Proof of Proposition \ref{prop:w1}]
  By Kantorovich-Rubenstein duality and Jensen's inequality, we have
  \begin{align*}
      \ee\left[W_1^M(P_n, \pp) \right] &\leq \ee\left[ \sup_{f \in \F} \frac 1n \sum_{i = 1}^n f(X_i) - \ee\left[f(X_i)\right]\right] \leq  \ee\left[ \sup_{f \in \F} \frac 1n \sum_{i = 1}^n f(X_i) - f(X_i')\right] =\ee\left[W_1^M(P_n, P_n')\right]
  \end{align*}
  where $\F$ is the class of functions on $M$ that are $1$-Lipschitz with respect to $d_M$.  Note that, by translation invariance, we may take the radius of the H{\"o}lder ball $\F$ to be $\Delta$.  By symmetrization and chaining,
  \begin{align*}
      \ee\left[ \sup_{f \in \F} \frac 1n \sum_{i = 1}^n f(X_i) - f(X_i')\right]  &\leq  2 \ee\left[ \sup_{f \in \F} \frac 1n \sum_{i = 1}^n \epsilon_i f(X_i)\right] \leq 2 \inf_{\delta > 0} \left[8 \delta + \frac{8 \sqrt 2}{\sqrt n} \int_\delta^\Delta \sqrt{\log N(\F, \norm{\cdot}_\infty, \epsilon)} d \epsilon     \right] \\
      &\leq \inf_{\delta > 0} \left[ 8 \delta + \frac{8 \sqrt 2}{\sqrt n}\int_\delta^\Delta \sqrt{3 \log \left(\frac{2 \Delta}{\epsilon}\right) \frac{d \vol M}{\omega_d} \left(\frac \pi 2\right)^d}\left(\frac{2}{\epsilon}\right)^{\frac d2} d \epsilon  \right] \\
      &\leq 2 \Delta \inf_{\delta > 0} \left[8 \delta + \frac{8 \sqrt{6}}{\sqrt n}\sqrt{\frac{d \vol M}{\omega_d}} \left(\frac \pi 2\right)^{\frac d2} \sqrt{\log \frac 1\delta} \int_\delta^1 \left(\frac{\Delta}{\epsilon}\right)^{- \frac d2} d \epsilon  \right]
  \end{align*}
  where the last step comes from Corollary \ref{cor:mfldHolderentropy} and noting that after recentering, $\F$ contains functions $f$ such that $\norm{f}_{L^\infty(M)} \leq \Delta$ and $\norm{\nabla f}_{L^\infty(M)} \leq 1$.  Setting 
  \begin{equation*}
    \delta = \frac \pi 2 \left(\frac{d \vol M}{n \omega_d \Delta^d}\right)^{\frac 1d}
  \end{equation*}
  gives
  \begin{equation*}
      \ee\left[ W_1^M(P_n, P_n')\right] \leq C \left(\frac{ \vol M}{n \omega_d}\right)^{\frac 1d} \sqrt[]{\log\left(\frac{n \omega_d \Delta^d}{d \vol_M}\right)}
  \end{equation*}
  for some $C \leq 48$, which concludes the proof.
\end{proof}

\begin{proof}[Proof of Theorem \ref{thm:gans}]
  By bounding the supremum of sums by the sum of suprema and the construction of $\widehat{\mu}$,
  	\begin{align*}
			d_{\beta, B}(\widehat{\mu}, \pp) &\leq  d_{\beta, B}(\widehat{\mu}, \widetilde{P}_n) + d_{\beta, B}(\widetilde{P}_n, \pp) \leq \inf_{\mu \in \mathcal{P}} d_{\beta, B}(\mu, \widetilde{P}_n) + d_{\beta, B}(\widetilde{P}_n, \pp) \\
			&\leq \inf_{\mu \in \mathcal{P}} d_{\beta, B}(\mu, \pp) + 2 d_{\beta, B}(\widetilde{P}_n, \pp) \\
			&\leq \inf_{\mu \in \mathcal{P}} d_{\beta, B}(\mu, \pp) + 2 d_{\beta, B}(\widetilde{P}_n, P_n) + 2 d_{\beta, B}(P_n, \pp).
		\end{align*}
		Taking expectations and applying Lemma \ref{lemma:holderipm} bounds the last term.  The middle term can be bounded as follows:
		\begin{align*}
		    d_{\beta, B}(\widetilde{P}_n, P_n) &= \sup_{f \in C_B^\beta(\Omega)} \frac 1n \sum_{ i = 1}^n f(X_i) - f(\widetilde{X}_i) \leq \sup_{f \in C_B^\beta(\Omega)} \frac 1n \sum_{i = 1}^n f(X_i) - f(X_i + \eta_i) + 2 B \epsilon \\
		    &\leq \sup_{f \in C_B^\beta(\Omega)} \frac 1n \sum_{i = 1}^n B \norm{\eta_i} + 2 B \epsilon
		\end{align*}
		where the first inequality follows from the fact that if $f \in C_B^\beta(\Omega)$ then $\norm{f}_\infty \leq B$ and the contamination is at most $\epsilon$.  The second inequality follows from the fact that $f$ is $B$-Lipschitz.  Taking expectations and applying Jensen's inequality concludes the proof.
\end{proof}

\begin{proof}[Proof of Corollary \ref{cor:gans}]
  Applying Kantorovich-Rubenstein duality, the proof follows immediately from that of Theorem \ref{thm:gans} by setting $\beta = 1$, with the caveat that we need to bound $B$ and the Lipschitz constant separately.  The Lipschitz constant is bounded by $1$ by Kantorovich duality.  The class is translation invariant, and so $\abs{\norm{f}_\infty - \ee[f]} \leq 2 R$ by the fact that the Euclidean diameter of $\mathcal{S}$ is bounded by $2R$.  The result follows.
\end{proof}
\begin{lemma}
  Let $X$ be distributed uniformly on a centred ($\ell^2$) ball in $\rr^d$ of radius $R$.  Then,
  \begin{equation*}
    \ee\left[\log \frac{R}{\norm{X}}\right] = \frac 1d.
  \end{equation*}
\end{lemma}
\begin{proof}
  Note that by scaling it suffices to prove the case $R = 1$.  By changing to polar coordinates,
  \begin{align*}
    \ee\left[\log\frac{1}{\norm{X}}\right] &= \frac{\int_{S^1} \int_0^1 \left(\log \frac 1r\right)r ^{d - 1}d r d \theta}{\int_{S^1} \int_0^1 r ^{d - 1}d r d \theta} \\
    &= -d \int_0^1 \left(\log r\right)r ^{d - 1}d r.
  \end{align*}
  Substituting $u = \log r$ and applying integration by parts then gives
  \begin{align*}
    -d \int_0^1 \left(\log r\right)r ^{d - 1}d r &= \left[\frac{r^d}{d} - r^d \log r\right]\bigg|_{r=  0}^{r = 1} = \frac 1d
  \end{align*}
  as desired.
\end{proof}